\documentclass{article}\usepackage[]{graphicx}\usepackage[]{color}
\makeatletter
\def\maxwidth{ %
  \ifdim\Gin@nat@width>\linewidth
    \linewidth
  \else
    \Gin@nat@width
  \fi
}
\makeatother

\definecolor{fgcolor}{rgb}{0.345, 0.345, 0.345}

\usepackage{framed}
\makeatletter
 {\par\unskip\endMakeFramed%
 \at@end@of@kframe}
\makeatother

\definecolor{shadecolor}{rgb}{.97, .97, .97}
\definecolor{messagecolor}{rgb}{0, 0, 0}
\definecolor{warningcolor}{rgb}{1, 0, 1}
\definecolor{errorcolor}{rgb}{1, 0, 0}
\newenvironment{knitrout}{}{} 

\usepackage{alltt}

\usepackage{etoolbox}
\newtoggle{icmlformat}

\togglefalse{icmlformat}

\usepackage{times}
\usepackage{graphicx} 
\usepackage{subfigure}

\iftoggle{icmlformat} {
  \usepackage{natbib}   
}{
  \usepackage[margin=1.5in]{geometry}
  \usepackage[sort&compress, numbers]{natbib}
}

\usepackage{algorithm}
\usepackage{algorithmic}

\usepackage{hyperref}

\iftoggle{icmlformat} {
  \usepackage{icml2015} 
}

\usepackage{amsmath}
\usepackage{amsthm}
\usepackage{appendix}

\allowdisplaybreaks

\usepackage{amssymb}
\usepackage{bbm}

\newcommand{\app}[1]{Appendix~\ref{app:#1}}
\newcommand{\lem}[1]{Lemma~\ref{lem:#1}}
\newcommand{\prop}[1]{Proposition~\ref{prop:#1}}

\newcommand{\mysec}[1]{Section~\ref{sec:#1}}
\newcommand{\mysecs}[1]{Sections~\ref{sec:#1}}
\newcommand{\mysecss}[1]{\ref{sec:#1}}
\newcommand{\eq}[1]{Eq.~(\ref{eq:#1})}
\newcommand{\eqs}[1]{Eqs.~(\ref{eq:#1})}
\newcommand{\eqss}[1]{(\ref{eq:#1})}
\newcommand{\eqw}[1]{Eq.~(#1)}
\newcommand{\fig}[1]{Fig.~(\ref{fig:#1})}

\newcommand{\npp}{\tilde{\eta}} 
\newcommand{\npq}{\eta} 
\newcommand{\mpq}{m} 
\newcommand{\gauss}{\mathcal{N}} 
\newcommand{\truecov}{\Sigma} 
\newcommand{\vbcov}{V} 
\newcommand{\constant}{C} 

\theoremstyle{plain}
\newtheorem{theorem}{Theorem}[section]
\newtheorem{proposition}[theorem]{Proposition}
\newtheorem{lemma}[theorem]{Lemma}

\newcommand{\kl}{\textrm{KL}}
\DeclareMathOperator*{\argmin}{arg\,min}
\newcommand{\mbe}{\mathbb{E}}

\newcommand{\cov}{\textrm{Cov}}

\iftoggle{icmlformat} {  
  \icmltitlerunning{Covariance Matrices for Mean Field Variational Bayes}
} {}

\iftoggle{icmlformat} {
}{
  \title{Covariance Matrices and Influence Scores for Mean Field Variational Bayes}
  \author{
    Ryan Giordano\\
    Department of Statistics\\
    University of California, Berkeley\\
    Berkeley, CA 94720 \\
    \texttt{rgiordano@berkeley.edu}
    \and
    Tamara Broderick \\
    Department of EECS,\\
    Massachusetts Institute of Technology\\
    Cambridge, MA 02139\\
    \texttt{tbroderick@csail.mit.edu}
  }
}
\IfFileExists{upquote.sty}{\usepackage{upquote}}{}
\begin{document}

\iftoggle{icmlformat} {
  \twocolumn[
  \icmltitle{Covariance Matrices and Influence Scores for Mean Field Variational Bayes}

  \icmlauthor{Ryan Giordano}{rgiordano@berkeley.edu} 
  \icmladdress{Department of Statistics,
               University of California, Berkeley
               Berkeley, CA 94720}
  \icmlauthor{Tamara Broderick}{tbroderick@csail.mit.edu}
  \icmladdress{
  Department of Electrical Engineering and Computer Science,
  Massachusetts Institute of Technology
  Cambridge, MA 02139}
  \icmlkeywords{keywords go here}
  ]
}{
  \maketitle
}

\begin{abstract} 
Mean field variational Bayes (MFVB) is a popular posterior approximation
method due to its fast runtime on large-scale data sets. However, it is well
known that a major failing of MFVB is that it underestimates
the uncertainty of model variables (sometimes severely)
and provides no information about model variable
covariance. We develop a fast, general methodology for exponential families
that augments MFVB to deliver accurate uncertainty estimates for model
variables---both for individual variables and coherently across variables.
MFVB for exponential families defines a fixed-point equation in the means of
the approximating posterior, and our approach yields a covariance estimate by
perturbing this fixed point. Inspired by linear response theory, we call our
method linear response variational Bayes (LRVB). 
We also show how LRVB can be used to quickly calculate a measure
of the influence 
of individual data points on parameter point estimates.
We demonstrate the accuracy
and scalability of
our method by learning Gaussian mixture models for both simulated and real data.
\end{abstract}


\section{Introduction}\label{sec:intro}

With increasingly efficient data collection methods, scientists are
interested in quickly analyzing ever larger data sets. 
In particular, the promise of these large data sets is not simply to fit 
old models but instead to learn more nuanced patterns from data
than has been possible in the past.
In theory, the Bayesian paradigm promises exactly these desiderata.
Hierarchical modeling allows practitioners to capture complex relationships
between variables of interest. Moreover, Bayesian analysis allows practitioners
to quantify the uncertainty in any model estimates---and to do so coherently
across all of the model variables.

\emph{Mean field variational Bayes} (MFVB), a method for approximating
a Bayesian posterior distribution, has grown in 
popularity due to its fast runtime on large-scale data sets
\cite{blei:2003:lda, blei:2006:dp, hoffman:2013:stochastic}.
But it is well known
that a major failing of MFVB is that it gives underestimates
of the uncertainty of model variables that can be 
arbitrarily bad, even when approximating a simple multivariate Gaussian distribution
\citep{mackay:2003:information,bishop:2006:pattern,turner:2011:two},
%
and provides no information about how
the uncertainties in different model variables interact
\cite{wang:2005:inadequacy, bishop:2006:pattern, rue:2009:approximate, turner:2011:two}.
We develop a fast, general methodology for exponential families that augments MFVB
to deliver accurate uncertainty estimates for model variables---both for individual
variables and coherently across variables.
In particular, as we elaborate in \mysec{mfvb},
MFVB for exponential families defines a fixed-point equation in the means
of the approximating posterior, and our approach yields a covariance estimate
by perturbing this fixed point.
The perturbations of \emph{linear response theory}
have previously been applied for machine learning by \cite{kappen:1998:efficient}
and specifically for mean-field methods by \cite{opper:2003:variational}
and \cite{welling:2004:linear}.
Our contribution is to use exponential families to 
derive particularly simple and scalable formulas for covariance estimation and
to develop a method to quickly calculate \emph{influence scores},
which measure the influence 
of individual data points on parameter point estimates.
We call our method \emph{linear response variational Bayes} (LRVB).

We demonstrate the accuracy and scalability of our 
LRVB covariance estimates with experiments
that focus on finite mixtures of multivariate Gaussians,
which have historically been a sticking point for MFVB
covariance estimates \cite{bishop:2006:pattern,turner:2011:two}.
We employ simulated data as well as the MNIST handwritten digit data set \cite{mnist:lecun1998gradient}.
We show that the LRVB variance estimates are nearly identical to
those produced by a Markov Chain Monte Carlo (MCMC) sampler,
even when MFVB variance is dramatically underestimated.
For these mixture models, we show that LRVB gives
accurate covariance estimates orders of magnitude faster
than MCMC on a wide range of problems. We demonstrate both theoretically and
empirically that, for this Gaussian mixture model, LRVB scales linearly in the number
of data points and approximately quadratically in the dimension of the parameter
space.  Finally,
we show how LRVB allows fast computation
of the influence scores mentioned above.

\section{Mean-field variational Bayes in exponential families} \label{sec:mfvb}

Denote our $N$ observed data points by the $N$-long column vector $x$,
and denote our unobserved model parameters by $\theta$. Here, $\theta$ is a column
vector residing in some space $\Theta$; it has $J$ subgroups and
total dimension $D$. Our model is specified by a distribution of the
observed data given the model parameters---the likelihood $p(x | \theta)$---and
a prior distributional belief on the model parameters
$p(\theta)$. Bayes' Theorem yields the posterior $p(\theta | x)$.

MFVB approximates $p(\theta | x)$ by a factorized distribution of the
form $q(\theta) = \prod_{j=1}^{J} q(\theta_{j})$ such that the
Kullback-Liebler divergence $\kl(q || p)$ between $q$ and $p$ is
minimized:
\begin{align*}
  q^{*} &:= \argmin_{q} \kl(q || p)\\
  &= \argmin_{q} \mbe_{q} \left[ \log p(\theta | x) - \sum_{j = 1:J} \log q(\theta_{j}) \right].
\end{align*}
By the assumed $q$ factorization, the solution to this
minimization obeys the following fixed point equations \cite{bishop:2006:pattern}:
\begin{equation}
	\label{eq:kl_div}
	\log q^{*}_{j}(\theta_{j}) = \mbe_{q^{*}_{i}: i \in [J]\setminus j} \log p(\theta, x) + \constant.
\end{equation}

Here, and for the rest of the text, $\constant$ denotes a constant and
$[J] := \{1,\ldots, J\}$.
For index $j$, suppose that
$p(\theta_{j} | \theta_{i \in [J]\setminus j}, x)$ is in natural exponential family form:
\begin{equation}
  \label{eq:variational_exp_def}
  p(\theta_{j} | \theta_{i \in [J]\setminus j}, x) = \exp(\npp_j^{T} \theta_{j} - A_j(\npp_j))
\end{equation}
with local natural parameter $\npp_j$ and local log partition function $A_j$.
Here, $\npp_j$ may be a function of $\theta_{i \in [J]\setminus j}$ and $x$.
If the exponential family assumption above holds for every index $j$,
then we can write $\npp_{j}$ as a sum of products of components 
of each $\theta_k$ vector:
\begin{equation}
  \label{eq:nat_param_sum_of_theta_prods}
\npp_{j} =
  \sum_{r \in R_j} G_r \prod_{k\in[J]\setminus j} \theta_{kr_k},
\end{equation}
where  $G_r$ is a $D_j$-sized column vector and $\theta_{kr_k}$ is a scalar.
Here, $r$ is a vector of length $J-1$. Each entry $r_k$ of $r$ is either
$\emptyset$ or an index in $[D_k]$. If $r_k = \emptyset$, then $\theta_{kr_k} = 1$;
otherwise, $\theta_{kr_k}$ is the $r_k$th element of the vector $\theta_k$. 
This notation scheme guarantees that each product contains at most one
factor from the vector $\theta_k$ for each index $k$.
In particular, the log likelihood is linear
in every vector $\theta_j$.  This property of the log likelihood
is guaranteed by \eq{variational_exp_def}.  \app{exp_fams}
contains further details and a proof of \eq{nat_param_sum_of_theta_prods}.

It follows from \eqs{kl_div}, \eqss{variational_exp_def},
\eqss{nat_param_sum_of_theta_prods}, and the assumed factorization of $q^{*}$ that
$\log q^{*}_{j}(\theta_{j})$ has the form
\begin{equation} \label{eq:exp_approx_marg}
          \left( \sum_{r \in R_j} G_r \prod_{k \in [J] \setminus j}
                    \mbe_{q^{*}_r} [\theta_{kr_k}] \right)^{T} \theta_{j} + \constant.
\end{equation}
We see that $q^{*}_{j}$ is in the same
exponential family form
as $p(\theta_{j} | \theta_{i \in [J]\setminus j}, x)$.
Let $\npq_{j}$ denote the natural
parameter of $q^{*}_{j}$, and denote the mean parameter of
$q^{*}_{j}$ as $\mpq_{j} :=  \mbe_{q^{*}_{j}} \theta_j$. We see from
\eq{exp_approx_marg} that
$$
\npq_{j} = \sum_{r \in R_j} G_r \prod_{k\in[J]\setminus j}\mpq_{kr_k}.
$$
Since $\mpq_{j}$ is a function of $\npq_{j}$,
we have the fixed point equations
$\mpq_{j} := M_{j}(\mpq_{i \in [J]\setminus j})$ for mappings $M_{j}$ across $j$ and
$$
\mpq := M(\mpq)
$$
for the vector of mappings $M$.

\section{Linear response covariance estimation} \label{sec:lr}

Let $\vbcov$ denote the covariance matrix of $\theta$ under the
factorized variational distribution $q^{*}(\theta)$, and let $\truecov$
denote the covariance matrix of $\theta$ under the true distribution,
$p(\theta | x)$:
$$
\vbcov := \cov_{q^{*}} \theta,
\quad \quad
\truecov := \cov_{p} \theta.
$$
$\vbcov$ may be a poor
estimate of $\truecov$, even when $\mpq \approx \mbe_{p} \theta$, i.e. when the marginal means match well
 \cite{mackay:2003:information,wang:2005:inadequacy, bishop:2006:pattern, rue:2009:approximate, turner:2011:two}.
Our goal is to use the MFVB solution and the techniques of
linear response theory
\cite{kappen:1998:efficient,opper:2003:variational,welling:2004:linear}
to construct an improved estimate for $\truecov$.

Define $p_{t}(\theta|x)$ such that its log is a linear
perturbation of the log posterior:
\begin{equation} \label{eq:perturbed_dens}
  \log p_{t}(\theta | x) = \log p(\theta | x) + t^{T} \theta - \constant(t),
\end{equation}
where $\constant(t)$ is a constant in $\theta$.  If we assume that $p(\theta | x)$
is a probability distribution with natural parameters in the interior of the
feasible space, then $p_{t}(\theta | x)$ is a probability distribution for
any $t$ in an open ball around $0$.  Since $\constant(t)$ normalizes
the $p_{t}(\theta | x)$ distribution, it is in fact the cumulant
generating function of $p(\theta | x)$. Further, every
(perturbed) conditional distribution 
$p_{t}\left(\theta_{j} | \theta_{i \in [J]\setminus j}, x\right)$
is in the same exponential family as every (unperturbed) conditional
distribution $p\left(\theta_{j} | \theta_{i \in [J]\setminus j}, x\right)$ 
by construction. So, for each $t$, we have mean field 
variational approximation $q_{t}^{*}$ with
marginal means $\mpq_{t,j} := E_{q_{t}^{*}} \theta_j$ and
fixed point equations
$\mpq_{t,j} = M_{t,j}(\mpq_{t,i \in [J]\setminus j})$ across $j$.
Thus, $\mpq_t = M_t(\mpq_t)$
as in \mysec{mfvb}.
Taking derivatives of the latter relationship with respect to $t$, we find
\begin{equation} \label{eq:mean_derivs}
  \frac{d \mpq_t}{d t^T}
    = \frac{\partial M_t}{\partial \mpq_t^T} \frac{d \mpq_t}{d t^T}
      + \frac{\partial M_t}{\partial t^T}.
\end{equation}
In particular, note that $t$ is a vector of size $D$ (the total dimension
of $\theta$), and $\frac{d \mpq_t}{d t^T}$, e.g., 
is a matrix of size $D \times D$ with $(a,b)$th entry equal to the scalar $d\mpq_{t,a} / dt_{b}$.

Since $q_t^{*}$ is the MFVB approximation for the perturbed
posterior $p_t(\theta|x)$,
we may hope that $\mpq_t = E_{q_{t}^{*}} \theta$ is close to the 
perturbed-posterior mean $\mbe_{p_t} \theta$.  The practical success of
MFVB relies on the fact that this approximation is often good in practice.
To derive interpretations of the individual terms in \eq{mean_derivs},
we assume that this equality of means holds,
but we indicate where we use this assumption with
an approximation sign: $\mpq_t \approx \mbe_{p_t} \theta$.
The full derivations of the following equations are given in \app{lr_derivs}.
\begin{align}
  \label{eq:main_dm_dt}
  \frac{d \mpq_t}{d t^T}
    &\approx \frac{d}{dt^T} \mbe_{p_t} \theta = \truecov_{t} \\
  \label{eq:main_dM_dt}
  \frac{\partial M_t}{\partial t^T}
    &= \frac{\partial}{\partial t^T} \mbe_{q_t^*} \theta = \vbcov_{t} \\
  \label{eq:main_dM_dm}
  \frac{d M_t}{d \mpq_t^T}
    &= \vbcov_{t} H_t,
\end{align}
where $\truecov_{t}$ is the covariance matrix of $\theta$ under $p_{t}$,
$\vbcov_{t}$ is the covariance matrix of $\theta$ under $q_t^*$, and 

\begin{equation*}
 H_t := E_{q^{*}} \left(\frac{ \partial^2  \log p_{t}(\theta | x)}
                       { \partial \theta \partial \theta^T}\right)
\end{equation*}

Then substituting \eqs{main_dm_dt}, \eqss{main_dM_dt}, and \eqss{main_dM_dm} into \eq{mean_derivs},
evaluating at $t=0$, and writing $H$ for $H_{0}$ and $V$ for $V_{0}$, we find
\begin{eqnarray}
  \hat{\truecov} &:=& \left. \frac{d \mpq_t}{d t^T} \right|_{t=0} \approx \truecov \nonumber \\
  \hat{\truecov} &=& \vbcov H \hat{\truecov} + \vbcov \Rightarrow \nonumber \\
  \hat{\truecov} &=& (I - \vbcov H)^{-1} \vbcov \label{eq:lrvb_est}
\end{eqnarray}
Thus, we call $\hat{\truecov}$ the LRVB estimate of the true posterior covariance $\truecov$.

\subsection{Exactness of multivariate normal and SEM} \label{sec:mvn}

Consider approximating a multivariate normal posterior distribution $p(\theta|x)$
with MFVB.
This case arises, for instance, given a multivariate normal likelihood
with fixed covariance $S$ and an improper uniform prior on the mean parameter $\mu$:
$$
  p(x | \mu) = \prod_{n=1:N} \gauss(x_n | \mu, S) \; \textrm{ and } \;
  q^{*}(\mu)=\prod_{j=1:J} q^{*}_j(\mu_j)
$$
Here, $\gauss$ represents the multivariate normal
distribution, and the total dimension $D$ of $\mu$ is equal to the number of components $J$.
So $\mu$ is a $J$-length vector for $J > 1$ with elements
$\mu_1, ..., \mu_J$, and $S$ is a known $J \times J$ positive definite matrix.
Our variational approximation, $q^{*}$, is given by the factorized distribution
over mean components.

In this case, it is well known that the MFVB posterior means are correct, but the
marginal variances are underestimated if $S$ is not diagonal.
This fact is often used to illustrate the shortcomings of MFVB
\cite{mackay:2003:information,wang:2005:inadequacy,bishop:2006:pattern,turner:2011:two}.

However, since the posterior means are correctly estimated,
the LRVB approximation in \eq{lrvb_est} is in fact an equality.
That is, for the posterior
location of a multivariate normal with known covariance,
\eq{lrvb_est} is not an approximation, and
$\hat\Sigma = \frac{d \mpq_t}{d t^T} = \truecov$ exactly.
A detailed proof of this fact can be found in \app{SEM}.

This result draws a connection between LRVB
and the ``supplemented expectation-maximization'' (SEM)
method of \cite{meng:1991:using}.  SEM is an asymptotically
exact covariance correction for the EM algorithm that transforms
the full-data Fisher information matrix into the observed-data Fisher
information matrix using a correction that is formally similar to
\eq{lrvb_est}.  In this sense, SEM is a
frequentist perspective on a special case of
the LRVB correction when the amount of data goes to infinity.
More details can be found in \app{SEM}.

\section{Scaling} \label{sec:scaling_formulas}


\eq{lrvb_est} requires the inverse of a matrix as large
as all the unknown natural parameters in the posterior $p(\theta | x)$, which
generally includes both main parameters and nuisance parameters.  In
many applications,
the number of nuisance parameters may be very large. For example,
in the finite mixture of normals
model below (\mysec{experiments}), there is an indicator variable for
the cluster assignment for each data point. If we treat these
variables as nuisance parameters,
the number of nuisance parameters grows with the number of data points $N$.
As a result, directly computing the matrix inverse in \eq{lrvb_est}
may be impractical.

However, since the variational covariance
$\vbcov$ is block diagonal and $H$ is often sparse, one may
be able to use Schur complements to efficiently calculate sub-matrices of
$\hat\Sigma$.
Suppose that our full parameter space, $\theta$, can be divided into a
small number of variables of primary interest, called $\alpha$,
and a large (and possibly growing) number of nuisance variables, $z$:
$$
\theta  = \left(\begin{array}{c}
\alpha\\
z
\end{array}\right),
\quad \quad
\truecov  =
\left[\begin{array}{cc}
\truecov_{\alpha} & \truecov_{\alpha z}\\
\truecov_{z\alpha} & \truecov_{z}
\end{array}\right].
$$
We can similarly define partitions for $H$ and $V$.  Assume also
that we have the usual mean field factorization of the variational
approximation: $q^*(\alpha,z) = q^*(\alpha)q^*(z)$, so that
$\vbcov_{\alpha z} = 0$.  (The variational
distributions may factor further as well.) 
We calculate the Schur complement of $\hat{\Sigma}$ in \eq{lrvb_est}
with respect to its $z$th 
component to find that 
\begin{align}
  \label{eq:nuisance_lrvb_est}
&\hat{\truecov}_{\alpha} = \\
&\quad  ( I_{\alpha} - V_{\alpha}H_{\alpha} -
  V_{\alpha}H_{\alpha z} \left(I_{z} - V_{z}H_{z})^{-1}
  V_{z}H_{z\alpha}\right)^{-1} V_{\alpha}
\nonumber
\end{align}
Here, $I_\alpha$ and $I_z$ refer to $\alpha$- and $z$-sized identity
matrices, respectively.  A detailed derivation can be found in
\app{lr_derivs}.  In cases where
$\left(I_{z} - V_{z}H_{z}\right)^{-1}$
can be efficiently calculated, \eq{nuisance_lrvb_est}
involves only an $\alpha$-sized inverse.
A finite mixture of Gaussians model, which we describe
in \mysec{experiments}, is one such case.

\section{Influence scores} \label{sec:influence}

Influence scores are a powerful tool from classical linear regression
that describe how much influence a particular data point has on
a modeled outcome.  They can be used, for example, to 
identify outliers and investigate the robustness of the linear model
\cite{chatterjee:1986:leverage,cook:1986:assessment}.
Analogously, it can be useful know how much Bayesian posterior means depend
on the values of individual data points. A number
of authors have proposed methods to measure the sensitivity of the
posterior distribution to perturbations or deletions of data points
both in linear models \cite{guttman:1992:bayesian,pena:1993:comparing}
and more generally
\cite{carlin:1991:expected,peng:1995:bayesian,zhu:2011:bayesian}.
LRVB gives a convenient formula to analytically
calculate the influence of individual data points as covariances between the $\theta$
vector and infinitesimal noise added to the data.

Consider the conditional expectation of a single parameter value,
$\theta_{i}$, as a function of a single data point, $x_{n}$. Specifically,
for notational convenience, define
\begin{eqnarray*}
m_{\theta_{i}}\left(x_{n}\right) & = &
  \mbe_{p}\left[\theta_{i}\vert x_{1},...,x_{n},...,x_{N}\right]
\end{eqnarray*}
One measure of the sensitivity of $\theta_{i}$ to $x_{n}$ is the
derivative of this function, 
$\frac{d}{dx_{n}}m_{\theta_{i}}\left(x_{n}\right)=m_{\theta_{i}}'\left(x_{n}\right)$.
We will refer to this derivative as an \emph{influence score}
for Bayesian models.

To draw a connection between this influence score and covariances,
imagine that our observations, $x$, are in fact slightly noisy
versions of the true data, $x^{*}$.  Specifically, our model
becomes
$$
  p(x | x^*, \theta) = p(x | x^*) p(x^* | \theta).
$$
In this new model, $x^{*}$ are unknown parameters, like $\theta$.
We assume our posterior beliefs about the true $x^{*}$ obey the following
assumptions\footnote{
Note that if each observation $x^*_n$ has only one sufficient statistic, the
perturbations can be treated as independent, and $S_{x}$ will be the identity.
However, if each observation $x^*_n$ has a vector of sufficient statistics,
$S_{x}$ must take that structure into account.  For example, if $x_n$
is drawn from a normal distribution centered at $x^*_n$,
it will have sufficient statistics
$x_n$ and $x_n^2$, which will be correlated with one another.  These correlations
will cause $S_{x}$ to be different from the identity in general.}:
\begin{eqnarray}
\label{eq:perturbation_assumptions}
\mbe\left(x^{*} \vert x\right) & = & x \nonumber\\
\cov\left(x^{*} \vert x\right) & = & \Sigma_{x} \\
S_{x} & := & \lim_{\epsilon \rightarrow 0} \frac{1} {\epsilon}\Sigma_{x} \nonumber\\
0 \ne S_{x} &<& \infty \nonumber\\
\textrm{Higher moments} & = &
    O\left(\epsilon^{p}\right)\textrm{, for }p>2 \nonumber
\end{eqnarray}

That is, the covariance matrix $\Sigma_{x}$ is
proportional to $\epsilon$.
Conditional on $x$, $m_{\theta_{i}}\left(x_{i}^{*}\right)$ is
a random variable that varies as the posterior belief about
$x_{i}^{*}$ varies around $x_{i}$.
By forming a Taylor expansion of $m_{\theta_{i}}'(x^*_{n})$
around $x_{n}$ we show 
for any data point $x_n$ and any parameter $\theta_i$ that:
\begin{align}\label{eq:influence_as_derivative}
m_{\theta_{i}}'\left(x_{n}\right) & =
  \lim_{\epsilon \rightarrow 0}\frac{1}{\epsilon}
   \cov\left(\theta_{i}, x_{n}^{*}\vert x\right)
\end{align}

(\app{influence_details} contains further details.)
That is, the limiting value of this covariance as $\epsilon \rightarrow 0$
can be used to estimate the influence
of observation $x_n$ on the mixture parameters in the spirit of classical
linear model influence scores from the statistics
literature.

Note that the covariances on the right hand side of \eq{influence_as_derivative}
are impossible to compute in naive MFVB,
since they involve correlations between distinct mean field components, and
difficult to compute using MCMC, since they require estimating a large number
of very small covariances with a finite number of draws.  However,
LRVB leads to a straightforward analytic expression for these
covariances.

To derive the LRVB influence scores,
we now assume that the parameter space of the posterior can be divided into
three types of
variables.  We have main and nuisance parameters,
called $\alpha$ and $z$ respectively as in \mysec{scaling_formulas},
and now also $x^{*}$, the unobserved data.
As before, we also assume that each has its own variational distribution,
i.e. $q^{*}(\theta) = q^{*}(\alpha)q^{*}(z)q^{*}(x^{*})$.  (The variational
distributions may factor still further.) We can write:
\begin{equation*}
\theta =  \left(\begin{array}{c}
\alpha\\
x^*\\
z
\end{array}\right) \quad \textrm{ and } \quad
\truecov = \left[\begin{array}{ccc}
\truecov_{\alpha} & \truecov_{\alpha x^*} & \truecov_{\alpha z}\\
\truecov_{x^*\alpha} & \truecov_{x^*} & \truecov_{x^*z}\\
\truecov_{z\alpha} & \truecov_{zx^*} & \truecov_{z}
\end{array}\right].
\end{equation*}

We use a similar partition for $V$ and $H$.  Recall that $\Sigma_x$ is
the result of an infinitesimal perturbation and nearly zero, so we express
our results in terms of $S_x$ in \eq{perturbation_assumptions}.
Let $\Sigma_\alpha$
denote the the ordinary LRVB covariance of $\alpha$ from \eq{nuisance_lrvb_est}.
The covariance between $\alpha$ and the infinitesimally perturbed $x$
then has the following formula:
\begin{align}
\label{eq:influence_score}
&\lim_{\epsilon \rightarrow 0} \frac{1}{\epsilon}\Sigma_{\alpha x^*} = \\
&\quad \Sigma_{\alpha}^{-1} ( V_{\alpha}H_{\alpha x^*} +
   V_{\alpha}H_{\alpha z}\left(I_{z}-V_{z}H_{z}\right)^{-1}V_{z}H_{zx^*})S_{x}
   \nonumber
\end{align}
This formula follows from the Schur inverse and
taking $\epsilon \rightarrow 0$.  Details can be found in
\app{influence_details}.

\section{Experiments} \label{sec:experiments}

Mixture models constitute some of the most popular models for MFVB
application \cite{blei:2003:lda, blei:2006:dp}
and are often used as an example
of where MFVB covariance estimates may go awry \cite{bishop:2006:pattern, turner:2011:two}.
We thus illustrate the efficacy of LRVB on the problem of approximating
the posterior when the likelihood is a finite mixture of multivariate Gaussians.

\subsection{Model} \label{sec:normal_mix_model}

We consider a $K$-component mixture of $P$-dimensional
multivariate normals with unknown component means, covariances,
and weights.
In what follows, the weight $\pi_k$ is the probability of the $k$th component,
$\gauss$ denotes the multivariate normal distribution,
$\mu_k$ is the $P$-dimensional mean of the $k$th component,
and $\Lambda_k$ is the $P \times P$
precision matrix of the $k$th component (so
$\Sigma_k := \Lambda_k^{-1}$ is the
covariance).  $N$ is the number of data points, and $x_{n}$
is the $n$th observed $P$-dimensional data point.
We employ the standard trick of augmenting the data generating
process with the latent indicator variables $z_{nk}$,
where $n=1,...,N$ and $k=1,...,K$, and
\begin{eqnarray*}
P(z_{nk} = 1) & = & \pi_k\\
z_{nk} = 1 & \Rightarrow & x_{n} \sim \gauss(\mu_k, \Lambda^{-1}_k)
\end{eqnarray*}
The full likelihood under this augmentation is
\begin{equation} \label{eq:normal_mixture_model}
  p(x | \pi, \mu, \Lambda, z) =
    \prod_{n=1:N} \prod_{k=1:K} \gauss(x_n | \mu_k, \Lambda^{-1}_k)^{z_{nk}}
\end{equation}
We assign independent variational factors to $\mu$, $\pi$, $\Lambda$, and
$z$.\footnote{
Unlike \mysec{mvn}, the variational posteriors for $\mu$
factor across components but not within components.  That is, for each
$k$, $q^{*}(\mu_k)$ is a multivariate (not a univariate) normal distribution.}
The $z$ variables are nuisance parameters.

Our goal is to estimate the covariance matrix of the
parameters $\log(\pi), \mu, \Lambda$ in the posterior
distribution $p(\pi, \mu, \Lambda | x)$ and to
estimate the influence of each data point $x_{n}$ on the posterior
means of $\log(\pi), \mu, \Lambda$ using 
LRVB (see \mysecs{lr} and \mysecss{influence}).

In addition to the standard MFVB covariance matrices, we
will compare the accuracy and speed of our estimates to
Gibbs sampling on the augmented model (\eq{normal_mixture_model})
using the function \texttt{rnmixGibbs}
from the R package \texttt{bayesm}.  Our LRVB implementation
relied heavily on linear algebra routines in
\texttt{RcppEigen}~\cite{rpackage:RcppEigen}.
We evaluate our results both on simulated data and on 
the MNIST data set~\cite{mnist:lecun1998gradient}.

\subsection{MNIST data set} \label{sec:mnist}

\newcommand{\MNISTDelta}{0.01}
\newcommand{\MNISTn}{12665}
\newcommand{\MNISTPerturbedN}{5}
\newcommand{\MNISTTestN}{2115}
\newcommand{\MNISTp}{25}
\newcommand{\MNISTTestAccuracy}{0.92}
\newcommand{\MNISTTestError}{0.08}
\newcommand{\MNISTInfluenceTime}{37}
\newcommand{\MNISTPerturbationTime}{20.7}

For a real-world example,
we applied LRVB to the unsupervised classification of two digits 
from the MNIST dataset of handwritten digits.
We first preprocess the MNIST dataset by performing principle component
analysis on the training data's centered pixel intensities 
and keeping the top $\MNISTp$ components.
For evaluation, the test data is projected onto the same
$\MNISTp$-dimensional subspace found using the training data.

We then treat the problem of
separating handwritten $0$s from $1$s as an unsupervised clustering
problem.  We limit the dataset to instances labeled as $0$
or $1$, resulting in $\MNISTn$ training and $\MNISTTestN$ test points.
We fit the training data
as a mixture of multivariate Gaussians.  Here, $K=2$, $P=\MNISTp$, and
$N=\MNISTn$.  Then, keeping the $\mu$, $\Lambda$, and $\pi$
parameters fixed, we calculate the expectations of the
latent variables $z$ in \eq{normal_mixture_model} for the test set.
We assign test set data point $x_n$ to whichever component has
maximum a posteriori expectation.  We count successful classifications
as test set points that match their cluster's majority label
and errors as test set points that are different from their cluster's
majority label.  By this measure, our test set error rate was
$\MNISTTestError$. We stress that we intend only to demonstrate
the feasibility of LRVB on a large, real-world dataset rather than
to propose practical methods for modeling MNIST.

\subsection{Covariance experiments} \label{sec:covariances}

\newcommand{\effsizecutoff}{500}
\newcommand{\simulationsize}{68}
\newcommand{\simulationn}{10000}
\newcommand{\simulationp}{2}
\newcommand{\simulationk}{2}
\newcommand{\simulationvbtime}{3.40}
\newcommand{\simulationmaptime}{16.98}
\newcommand{\simulationgibbstime}{306.97}

In this section, we check the covariances estimated with
\eq{lrvb_est} against a Gibbs sampler, which we treat as the ground
truth.\footnote{
The likelihood described in \mysec{normal_mix_model} is symmetric under
relabeling.  When the component locations and shapes have
a real-life interpretation, the researcher is generally
interested in the uncertainty of $\mu$, $\Lambda$, and $\pi$ for a
particular labeling, not the
marginal uncertainty over all possible re-labelings.  This poses
a problem for standard MCMC methods, and we restrict our simulations
to regimes where label switching did not occur in our Gibbs sampler.
The MFVB solution conveniently avoids this problem since the mean field
assumption prevents it from representing more than one mode of the
joint posterior.}

For simulations, we generated $N=\simulationn$ data points from
$K=\simulationk$ multivariate normal components in
$P=\simulationp$ dimensions.  MFVB is expected
to underestimate the marginal variance of $\mu$, $\Lambda$, and $\log(\pi)$
when the components overlap since that induces correlation in the
posteriors due to the uncertain classification of points between the
clusters.  These correlations are in violation of the MFVB assumption
and cause the MFVB posterior variances to be mis-estimated.

We performed $\simulationsize$ simulations, each of which had
at least $\effsizecutoff$ effective Gibbs
samples in each variable---calculated with the R tool \texttt{effectiveSize}
from the \texttt{coda} package \cite{rpackage:coda}.
We note that for each of the parameters $\log(\pi)$, $\mu$, and $\Lambda$,
both MH and MFVB produce posterior means close to the
ground truth MCMC values, so our key assumption in the
LRVB derivations of \mysec{lr} appears to hold.

Each point in \fig{SimulationStandardDeviations} represents the
a single parameter in a single simulation.
For example, each point on the $\Lambda$ graph represents the 
marginal standard deviation of a particular component of the $\Lambda$
matrix for both the Gibbs sample and an alternative method.
The first three graphs show the diagonal standard deviations,
and the final graph shows the off-diagonal covariances.  Note
that the final graph excludes the MFVB estimates since most
of the values are zero.

\fig{SimulationStandardDeviations} shows that the 
raw MFVB covariance estimates are often quite different from the
Gibbs sampler results, while the LRVB estimates match
the Gibbs sampler closely.  Although not shown, the results
on the MNIST dataset were as good.

In these simulations, on average LRVB took only $\simulationvbtime$ seconds,
whereas the Gibbs sampler took $\simulationgibbstime$ seconds.
We explore these timing tradeoffs in more detail in \mysec{scaling}.


\begin{knitrout}
\definecolor{shadecolor}{rgb}{0.969, 0.969, 0.969}\color{fgcolor}\begin{figure}[ht!]

{\centering \includegraphics[width=0.49\linewidth,height=0.49\linewidth]{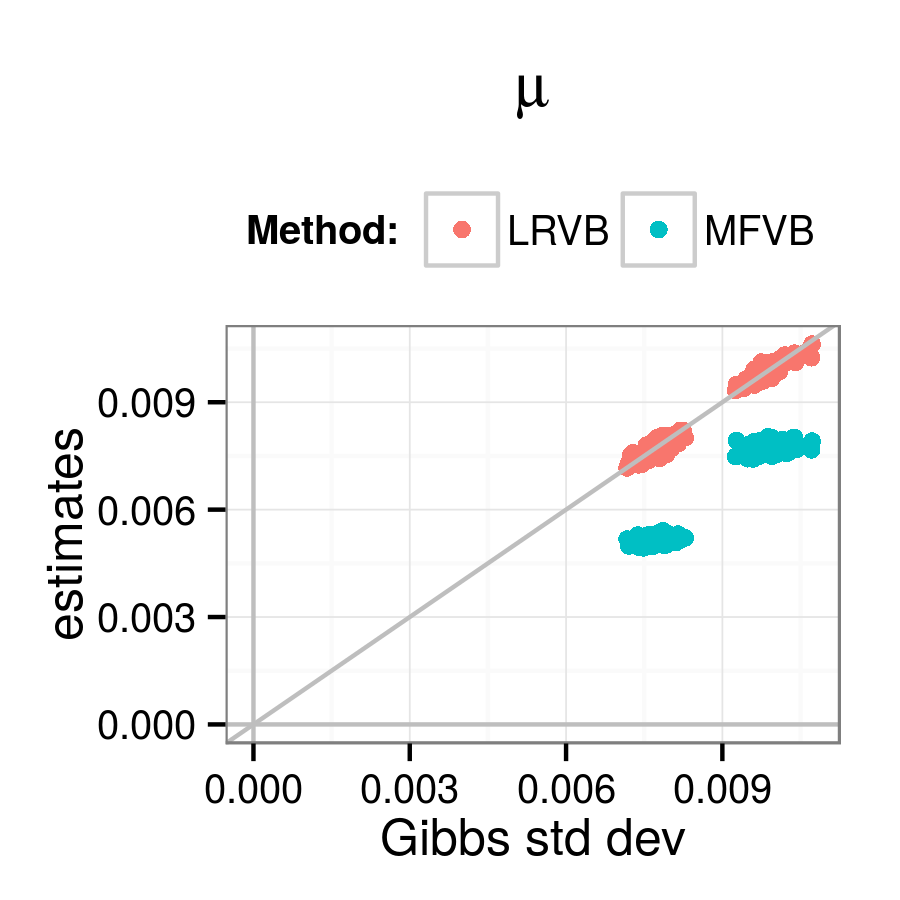} 
\includegraphics[width=0.49\linewidth,height=0.49\linewidth]{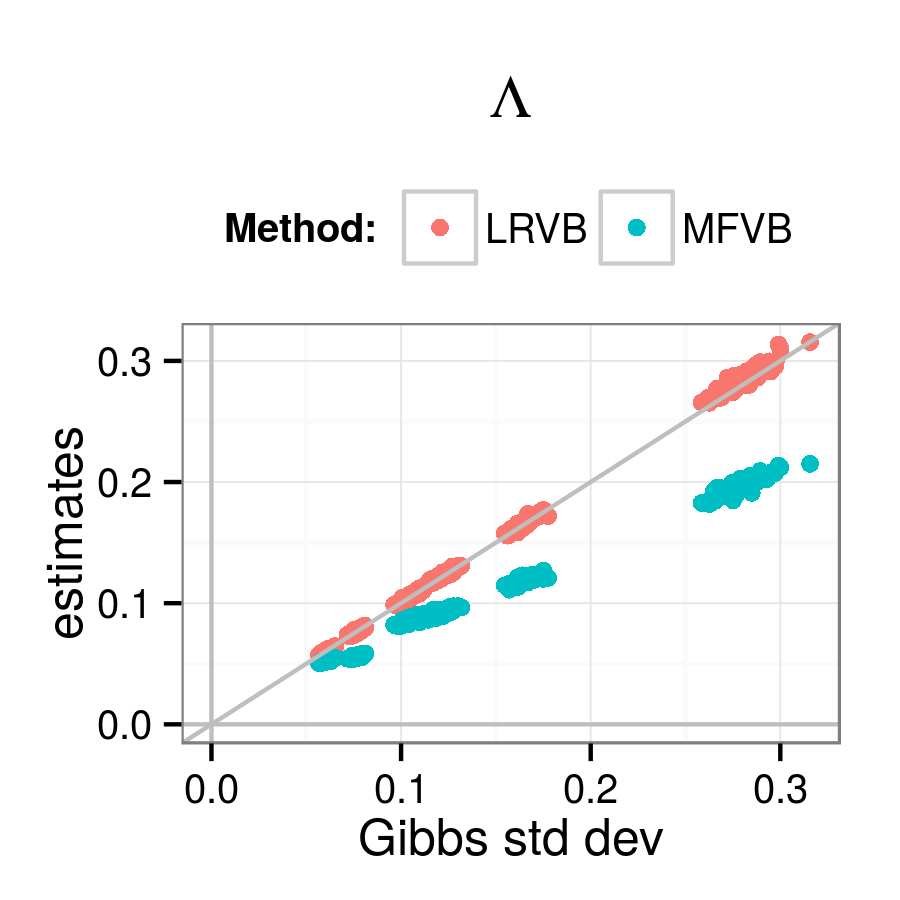} 
\includegraphics[width=0.49\linewidth,height=0.49\linewidth]{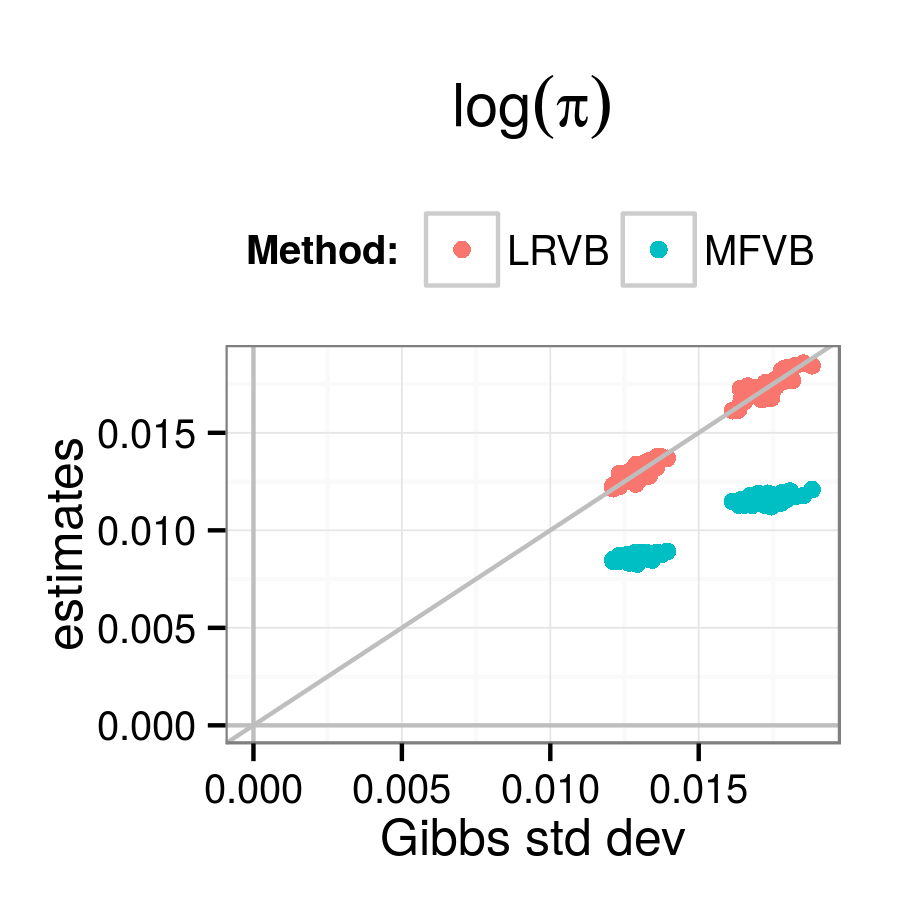} 
\includegraphics[width=0.49\linewidth,height=0.49\linewidth]{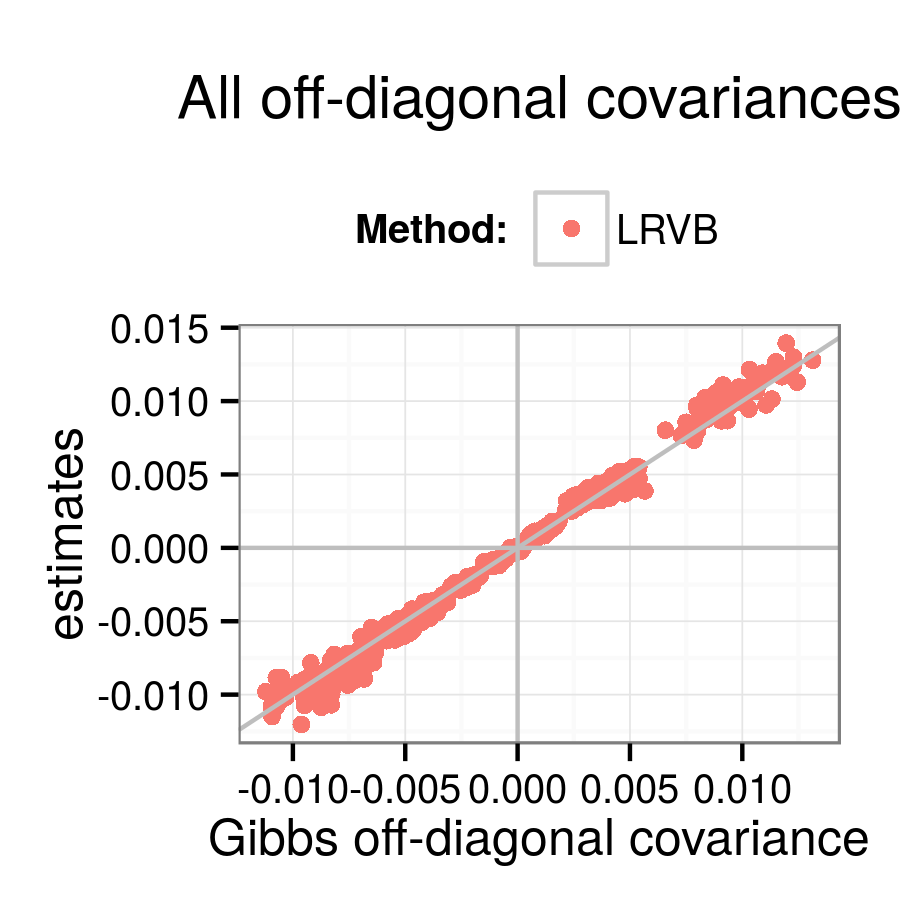} 

}

\caption[Comparison of estimates of the posterior covariance matrix on simulation data for each model parameter from Gibbs, MFVB, and LRVB methods]{Comparison of estimates of the posterior covariance matrix on simulation data for each model parameter from Gibbs, MFVB, and LRVB methods.  In the simulations, $N=\simulationn$ (data points), $K=\simulationk$ (components) and $P=\simulationp$ (dimensions).\label{fig:SimulationStandardDeviations}}
\end{figure}

\end{knitrout}

%

\subsection{Scaling experiments} \label{sec:scaling}

In this section we show that, for the finite mixture of
multivariate Gaussians model, \eq{lrvb_est} scales linearly with $N$
and polynomially in $K$ and $P$.  We also
use simulations to experimentally compare the scaling of
LRVB running times to Gibbs sampling estimates.
We show that LRVB is much faster than Gibbs for the
range of parameters we simulated, though Gibbs may be
preferable for very high-dimensional problems.

In the terms of \mysec{scaling_formulas}, $\alpha$ 
includes the sufficient statistics from $\mu$, $\pi$, and $\Lambda$,
and grows as $O(KP^2)$.
The sufficient statistics for the variational posterior of
$\mu$ contain the $P$-length vectors $\mu_k$, for each
$k$, and the $(P + 1) P / 2$ second-order products
in the covariance matrix $\mu_k \mu_k^T$.  Similarly, for each $k$,
the variational posterior of $\Lambda$ involves the
$(P + 1) P / 2$ sufficient statistics in the symmetric matrix
$\Lambda_k$ as well as the term $\log |\Lambda_k|$.  The
sufficient statistics for the posterior of $\pi_k$ are the $K$
terms $\log \pi_k$.\footnote{
Since $\sum_{k=1:K}\pi_k = 1$, using $K$
sufficient statistics involves one redundant parameter.
However, this does not violate any of the necessary assumptions
for \eq{lrvb_est}, and it considerably simplifies the calculations.
Note that though the perturbation argument of \mysec{lr}
requires the natural parameters of
$p(\theta | x)$ to be in the interior of the feasible space,
it does not require that the natural parameters of $p(x | \theta)$
be interior.}
This means that, minimally, \eq{lrvb_est}
will require the inverse of a matrix of size $O(KP^2)$.

The sufficient statistics for
$z$ have dimension $K \times N$.  In other words,
the number of nuisance parameters grows with the number of
data points, but $H_{z}=0$ for the multivariate normal
(\app{mvn_details} contains further details),
so we can apply \eq{nuisance_lrvb_est} to replace the
inverse of an $O(KN)$-sized matrix
with multiplication by an $O(KN)$-sized matrix.  Here,
$z$ conveniently corresponds directly to the $z$
in \mysec{scaling_formulas}.

Since a matrix inverse is cubic in the size of the matrix,
the worst-case scaling for LRVB is then $O(K^2)$ in $K$,
$O(P^6)$ in $P$ and $O(N)$ in $N$.

In our simulations, shown in \fig{ScalingGraphs}, we can see that,
in practice, LRVB scales linearly in $N$ and slightly less than
quadratically in $P$, which is much better than the theoretical worst case.
Note that the vertical axis, the time to run the algorithm,
is on the log scale.  At every value of $P$, $K$, and $N$ examined here,
calculating LRVB is much faster than Gibbs sampling.\footnote{
For numeric stability
we started the optimization procedures for MFVB at the true
values, so the time to compute the optimum in our simulations
was very fast and not representative of practice.
On real data, the optimization time will depend on the
quality of the starting point.
Consequently, the times shown for LRVB are only the 
times to compute the LRVB estimate.  The optimization times were
on the same order.
The Gibbs sampling time was linearly rescaled to the amount
of time necessary to achieve 1000 effective samples in the slowest-mixing
component of any parameter.}

\begin{knitrout}
\definecolor{shadecolor}{rgb}{0.969, 0.969, 0.969}\color{fgcolor}\begin{figure}[ht!]

{\centering \includegraphics[width=0.49\linewidth,height=0.49\linewidth]{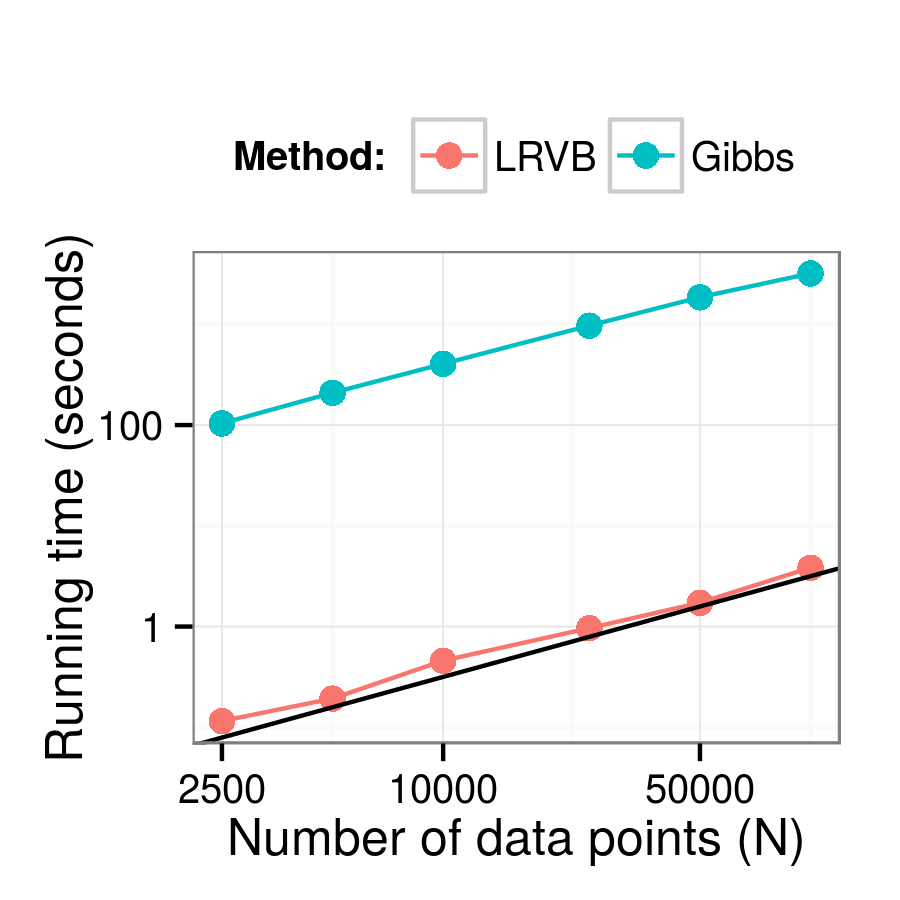} 
\includegraphics[width=0.49\linewidth,height=0.49\linewidth]{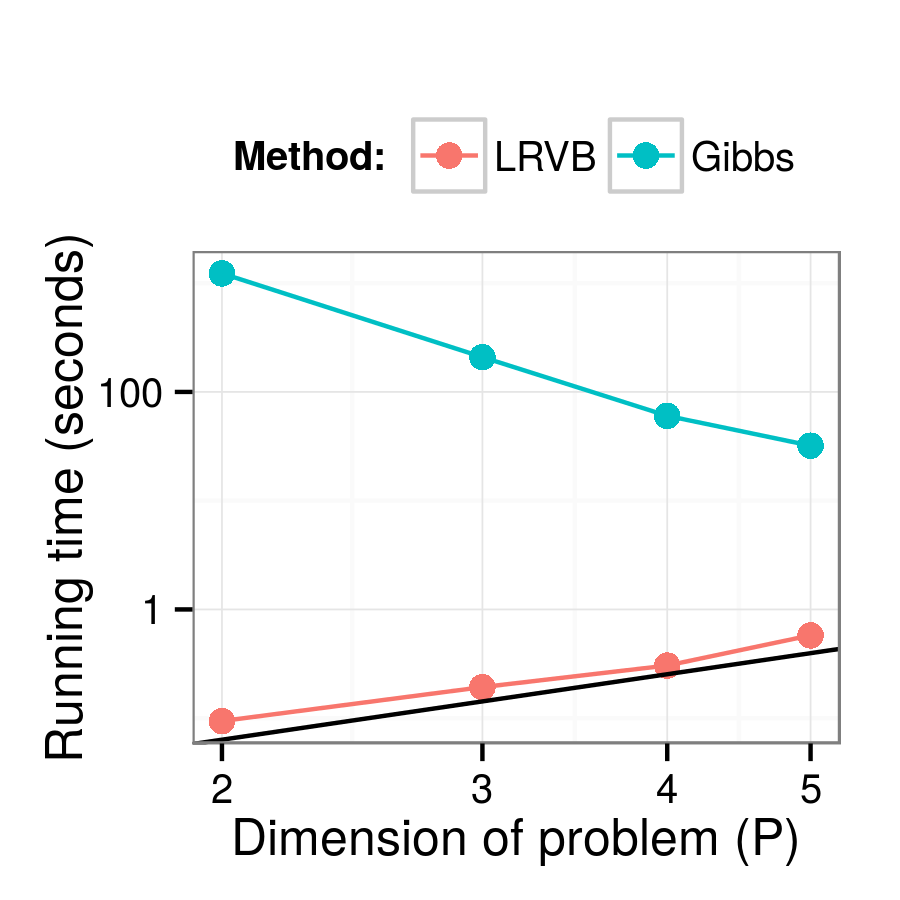} 
\includegraphics[width=0.49\linewidth,height=0.49\linewidth]{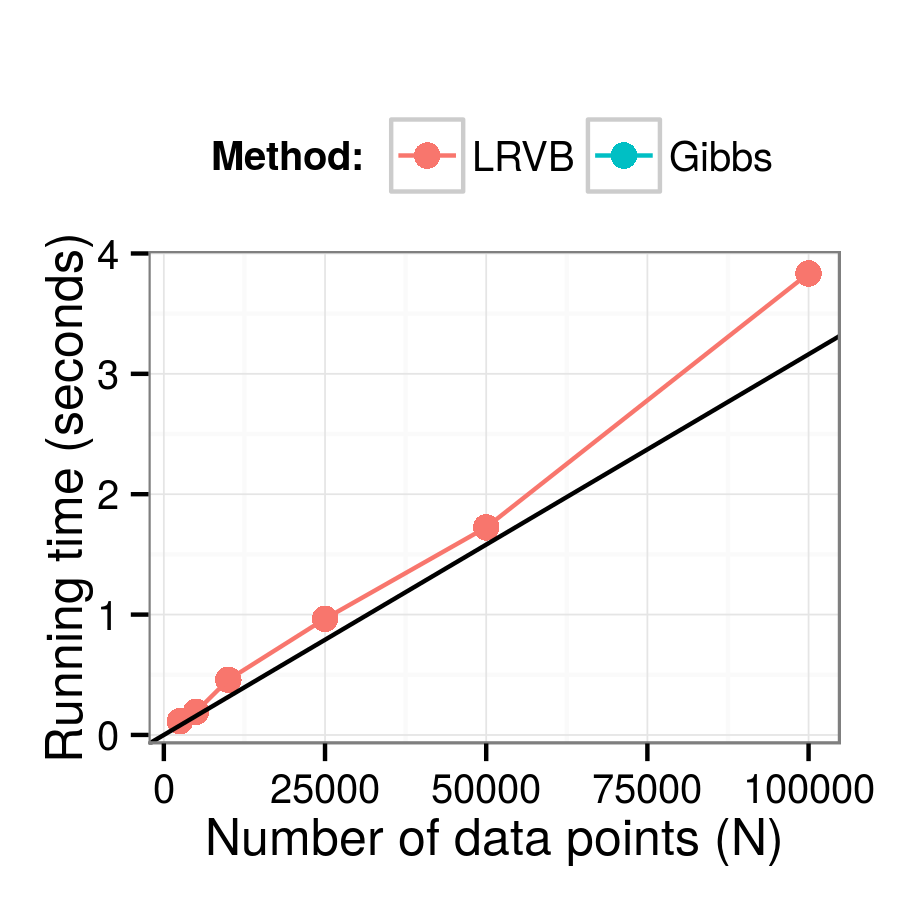} 
\includegraphics[width=0.49\linewidth,height=0.49\linewidth]{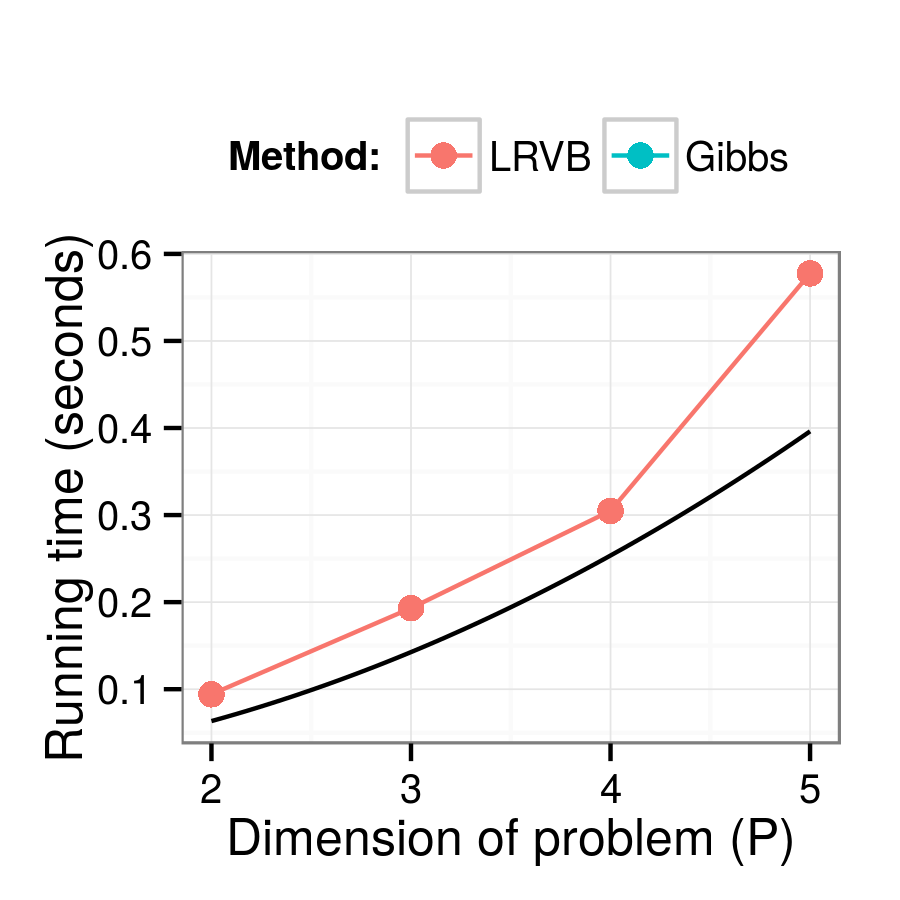} 

}

\caption[Scaling of LRVB and Gibbs on simulation data in both log and linear scales]{Scaling of LRVB and Gibbs on simulation data in both log and linear scales.  Before taking logs, the line in the (N) graph is $y\propto x$, and in the (P) graph, it is $y \propto x^2$.\label{fig:ScalingGraphs}}
\end{figure}

\end{knitrout}

\subsection{Influence score experiments} \label{sec:normal_sensitivity}

We now demonstrate the accuracy of the influence score
formula, \eq{influence_score}, on simulated data
and on the MNIST dataset.  We first compare \eq{influence_score}
to numeric derivatives.  We then look at some patterns in the
data that are made visible by having easy-to-calculate influence scores.

\subsubsection{Comparison with numeric derivatives}

\newcommand{\LowOverlapN}{10000}
\newcommand{\LowOverlapP}{2}
\newcommand{\LowOverlapK}{2}

In order verify that LRVB influence scores are correct for this model,
we manually perturbed each component of the data and
re-fit to find the new MFVB optimum.
In other words, we numerically compute the derivative in
\eq{influence_as_derivative}.

Our simulation used $N=\LowOverlapN$, $P=\LowOverlapP$,
and $K=\LowOverlapK$.  This was small enough to calculate the
influence score for every data point in every dimension.
To select sample data points for MNIST influence scores, we selected
$\MNISTPerturbedN$ representative data points
with different ranges of $\mbe_{q^{*}}z_n$
values so that some had their posterior probability concentrated
in only one component, and some that were uncertainly classified
between the two components.  We then computed the
LRVB influence scores from \eq{influence_score}.  For comparison,
we again manually perturbed each dimension of each data
point and re-optimized.

On MNIST, not counting the time to compute the initial LRVB covariance,
calculating the LRVB influence scores
took $\MNISTInfluenceTime$ seconds,
and the process of perturbing and re-optimizing
took $\MNISTPerturbationTime$ minutes.

The comparison between numeric differentiation and LRVB influence scores
for both a simulation (left) and MNIST (right)
is shown in \fig{InfluenceVsDerivativesGraphs}.
The influence scores obtained from the two
approaches are practically indistinguishable.
Though not shown, in both simulations and on MNIST, for each
$\mu$, $\Lambda$, and $\log(\pi)$
parameter, the two methods were as similar to one another as in \fig{InfluenceVsDerivativesGraphs}.

\subsubsection{Influence score data}

One can see interesting patterns in the data with influence scores.
Consider, for example, the simulated data, which is depicted in \fig{InfluenceLandscape}
and has moderately overlapping components.
The graph shows the effect on $\mu_{11}$ of perturbing the
$x_{n1}$ (horizontal)
coordinate of each datapoint.  One can see that the component's
mean is essentially determined by the points that are assigned to it.
Interestingly, points on the border between the two components reverse the
sign of their effect.  This is caused by changes in $\mbe_{q^{*}} z_n$ that more than
counterbalance their effects on $\mbe_{q^{*}} \mu$.

\begin{knitrout}
\definecolor{shadecolor}{rgb}{0.969, 0.969, 0.969}\color{fgcolor}\begin{figure}[ht!]

{\centering \includegraphics[width=0.49\linewidth,height=0.49\linewidth]{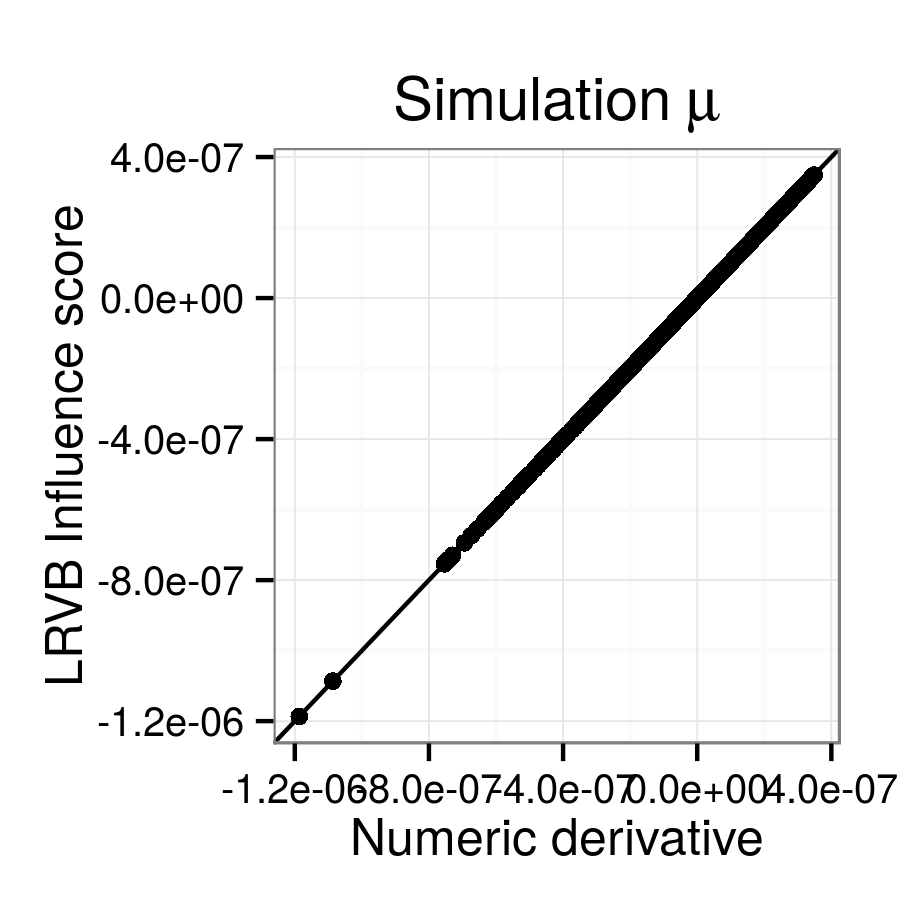} 
\includegraphics[width=0.49\linewidth,height=0.49\linewidth]{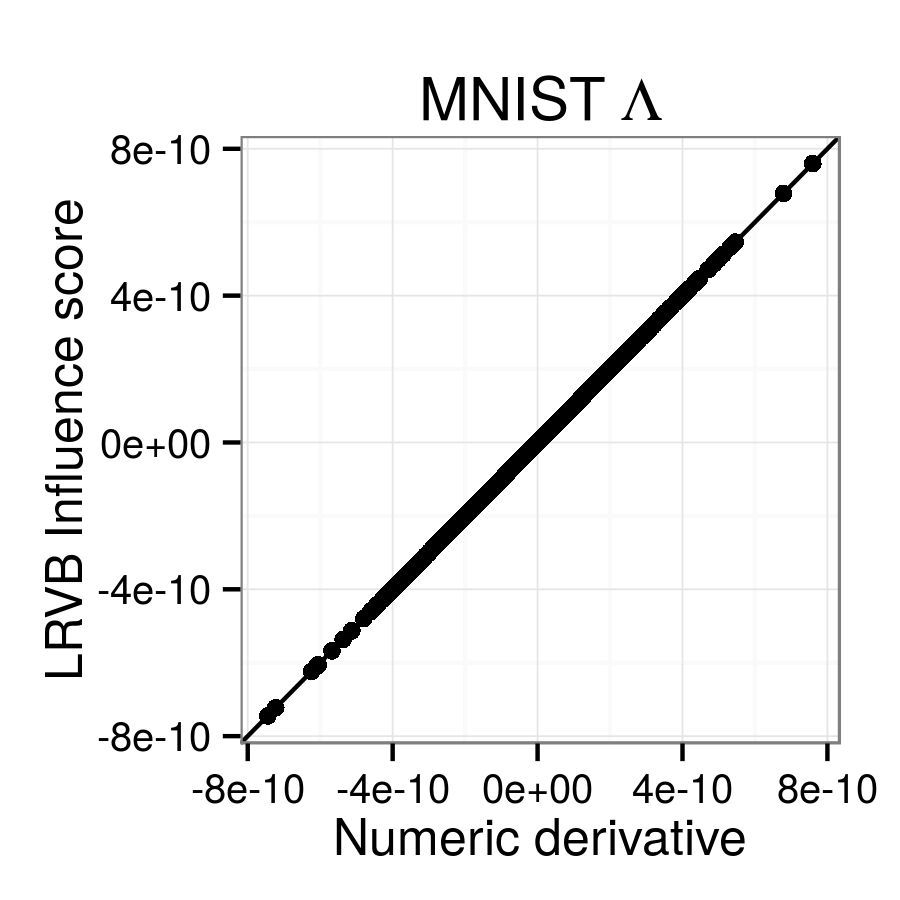} 

}

\caption[LRVB influence score accuracy]{LRVB influence score accuracy.  The simulated data uses all $x_n$, and the MNIST data uses five representative $x_n$.\label{fig:InfluenceVsDerivativesGraphs}}
\end{figure}

\end{knitrout}

As can be seen in the simulated data of \fig{InfluenceLandscape},
the situation becomes more complex when the components overlap even more.
Data points that are distant from a component center
have nearly as much influence as data points well within the component.

\begin{knitrout}
\definecolor{shadecolor}{rgb}{0.969, 0.969, 0.969}\color{fgcolor}\begin{figure}[ht!]

{\centering \includegraphics[width=1\linewidth,height=0.75\linewidth]{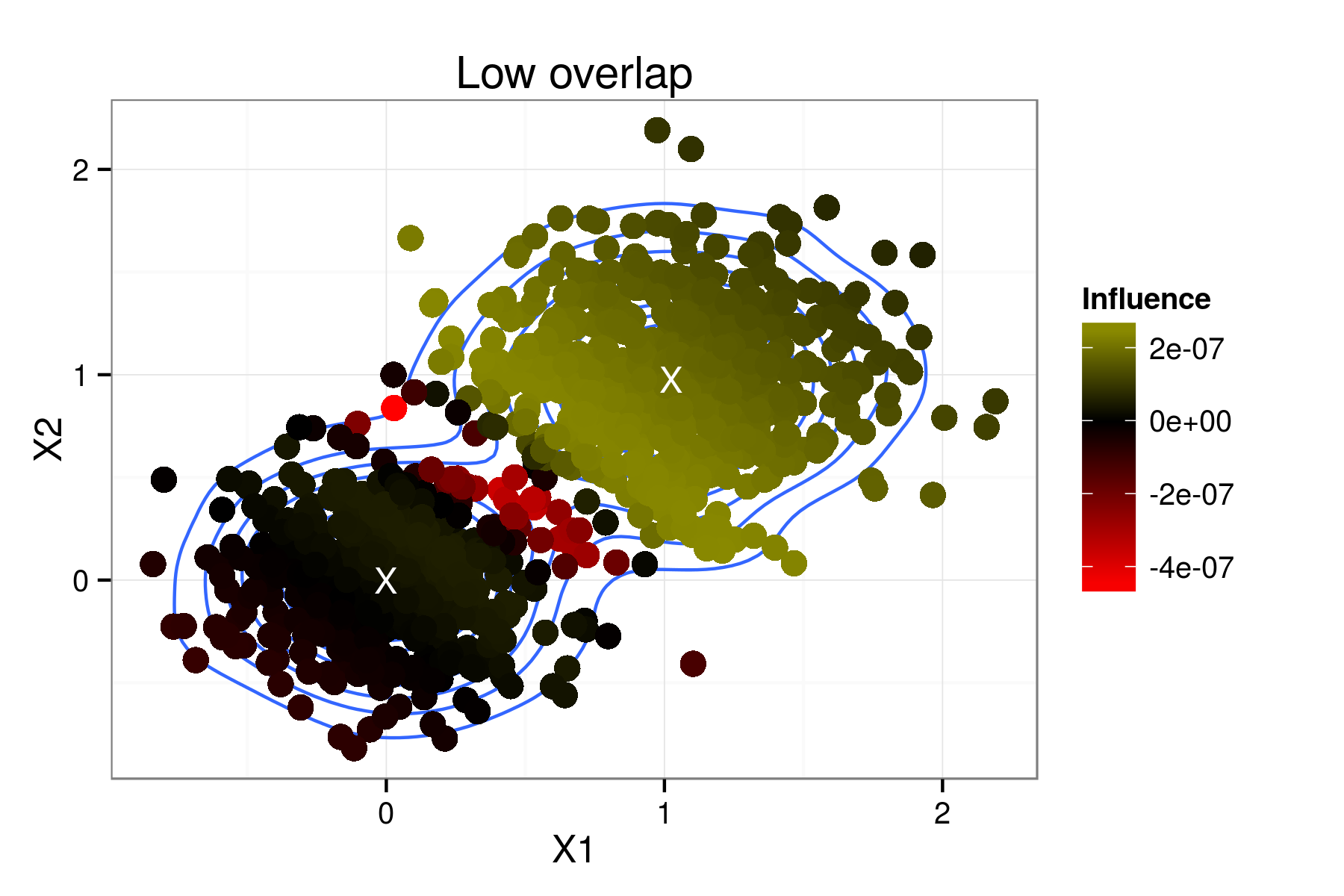} 
\includegraphics[width=1\linewidth,height=0.75\linewidth]{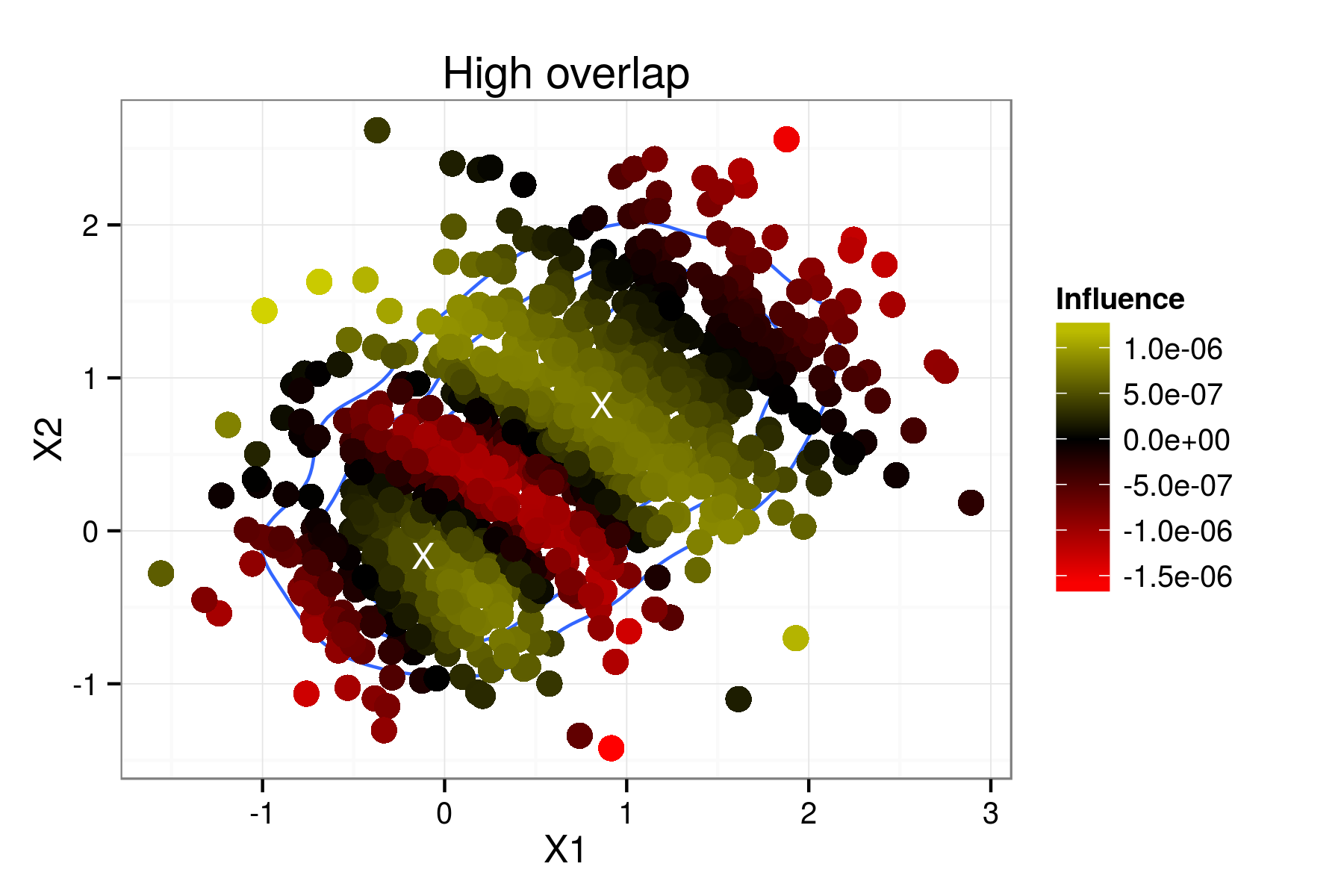} 

}

\caption[Influence scores for components with different amounts of overlap]{Influence scores for components with different amounts of overlap.  Each graph shows the influence of $x_{n1}$ on $\mu_{11}$, which is mean of the upper-right hand component.  (X) indicates a component posterior mean.\label{fig:InfluenceLandscape}}
\end{figure}

\end{knitrout}

Finally, we consider influence scores in the MNIST data set.
We selected $100$ data points with a 0 or 1 label uniformly at random.
In \fig{InfluenceLandscape}, we considered the influence of one particular dimension
of each data point on one particular dimension of each component mean.
Recall
from \mysec{mnist} that each data point and component mean
is $\MNISTp$-dimensional.
Now we wish to derive a single influence score for the 
effect of each data point vector-valued $x_n$ on each vector-valued component mean $\mu_k$.

To that end,
we define the influence of a data point $x_n$ on $\mu_k$ as the
directional derivative of $\|\mu_k\|_2^2$ with respect to $x_n$.
That is, we calculate the vector $\partial \| \mu \|_2^2 / \partial x_n$
using LRVB as described above. Then we compute
$$
\textrm{Influence of }x_n\textrm{ on }\mu_k :=
   \max_{\delta: \|\delta\| = 1}\left(
       \frac{\partial \|\mu_k\|_2^2}{ \partial x_n^T} \delta \right)
$$

The resulting influences
are plotted in
\fig{MNISTInfluenceDistribution}.  Each point corresponds to
a data point $x_n$. Each sub-figure corresponds to a different
component mean parameter. The horizontal axis value for 
point $x_n$ is the logit of
$\mbe_{q^*} z_{nk}$ (capped at $\pm15$), which measures the posterior
probability that $x_n$ came from component $k$.

The two components show
very different patterns.  Component $0$, the
mode with mostly handwritten zeroes, has much higher influence
amongst points that are classified within it than component $1$.

\begin{knitrout}
\definecolor{shadecolor}{rgb}{0.969, 0.969, 0.969}\color{fgcolor}\begin{figure}[ht!]

{\centering \includegraphics[width=0.49\linewidth,height=0.49\linewidth]{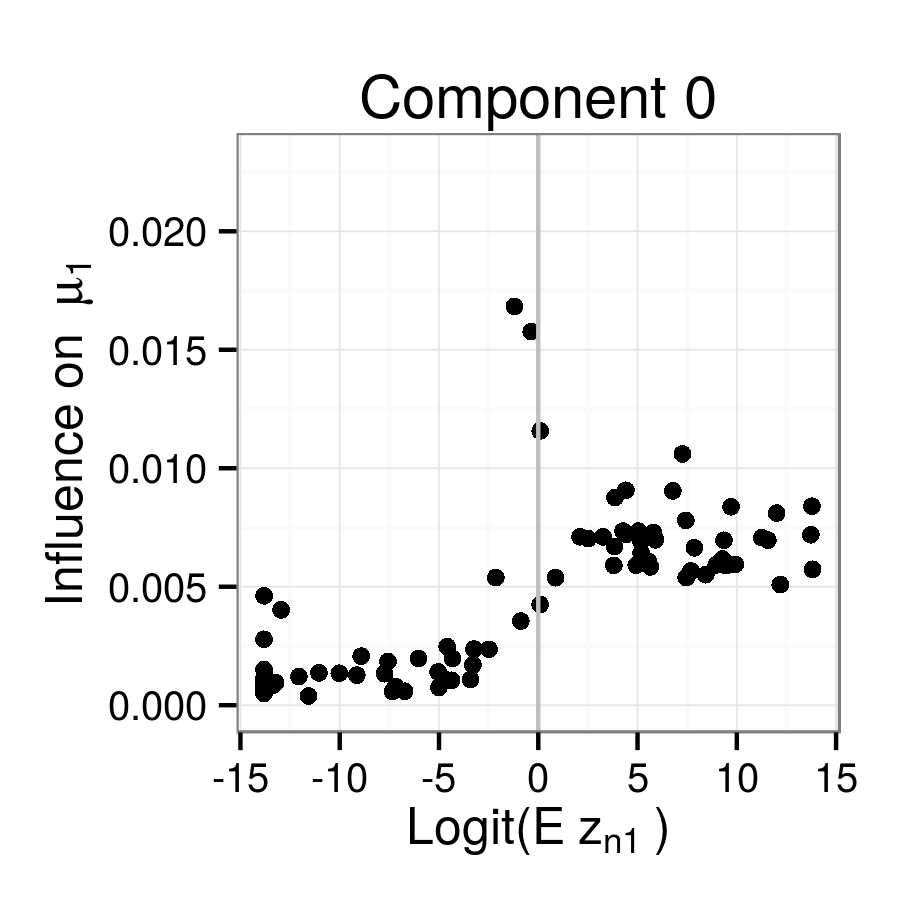} 
\includegraphics[width=0.49\linewidth,height=0.49\linewidth]{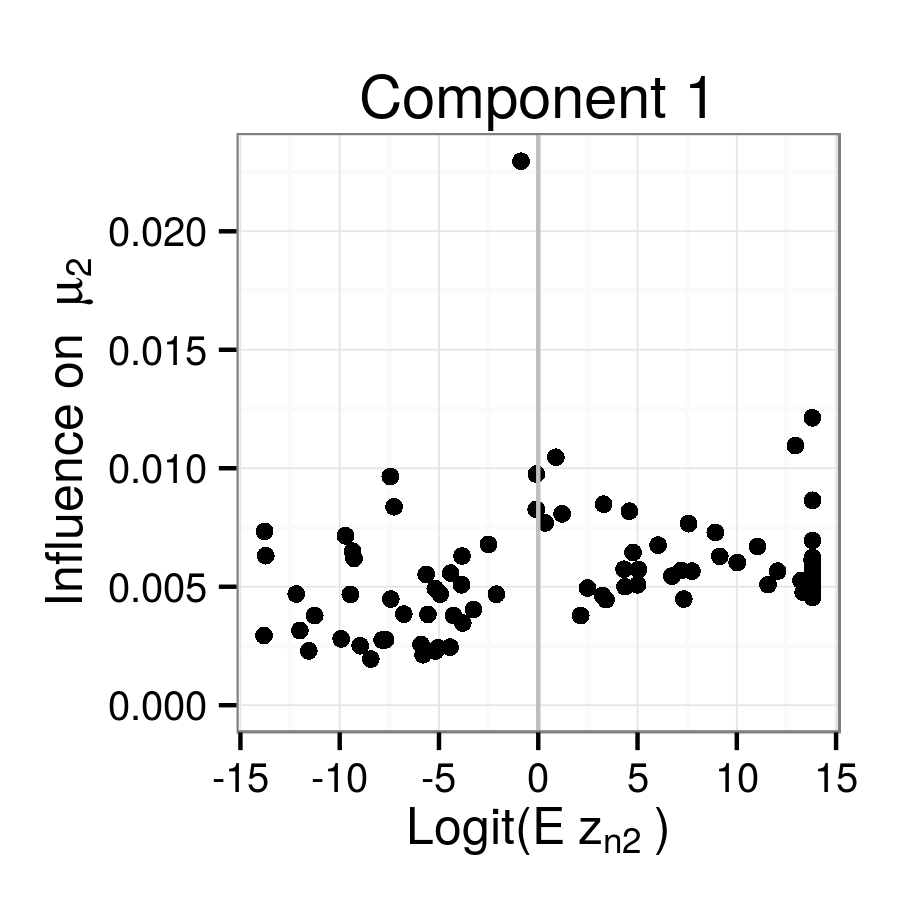} 

}

\caption[Different influence patterns for the two Gaussians in the MNIST dataset]{Different influence patterns for the two Gaussians in the MNIST dataset.  The 100 data points were chosen randomly.  The vertical axis shows the maximum directional derivative of $\mu_k$ with respect to changes in the data point, and the horizontal axis shows the (capped) logit posterior probability that the point came from that component.\label{fig:MNISTInfluenceDistribution}}
\end{figure}

\end{knitrout}

\section{Conclusion} \label{sec:conclusion}

The lack of accurate covariance estimates from the widely used
mean-field variational Bayes (MFVB) methodology has been a longstanding 
shortcoming of MFVB.
We have demonstrated that our method,
linear response variational Bayes (LRVB), 
augments MFVB to deliver
these covariance estimates in time that scales linearly
with the number of data points.
We have also shown how to use LRVB to quickly calculate influence scores, a measure
of the influence of each data point on posterior parameter means. 
Our experiments have focused on mixtures of multivariate Gaussians since these have
traditionally been used to illustrate the difficulties with MFVB covariance estimation.
We hope that in future work our results can be extended to more complex models,
including latent Dirichlet allocation
and Bayesian nonparametric models, where MFVB has proven its practical success.




\iftoggle{icmlformat} {
  \bibliographystyle{icml2015}
}{
  \bibliographystyle{unsrt} 
}
\bibliography{lrvb_icml}


\onecolumn
\appendix
\renewcommand{\appendixpagename}{Supplementary Material}
\appendixpage

\section{Derivations} \label{app:derivs}

\subsection{MFVB for conditional exponential families} \label{app:exp_fams}

First, we require some notation for indexing $\theta$.  Recall that the MFVB
assumption partitions the components of $\theta$ into $J$ groups according to
the factorization
$$
q(\theta) = \prod_{j=1}^J q_j(\theta_j) = \prod_{j=1}^J q(\theta_j)
$$
We follow the common abuse of notation of defining the variational functions
through their arguments by writing $q(\theta_j)$ for $q_j(\theta_j)$.

Each $\theta_j$ is be a $D_j$-dimensional vector, with $\sum_{j=1}^J D_j = D$,
where the whole $\theta$ vector has dimension $D$.  Here and below,
we will define the set $[J] := \{1,...,J\}$.

Let $R(\theta)$ denote a vector of length $\prod_{j\in[J]} (D_j + 1)$
that is the vector of all possible products of the form
$\prod_{j \in [J]}\theta_{ji_j}$, where for each $j$,
$i_j \in \{\emptyset, 1, ..., D_j\}$. 
That is, $R(\theta)$ is the
vector of all possible products of terms of $\theta$ where
with at most one term from each $\theta_j$,
and we define $\theta_{j\emptyset} := 1$.
Let $R_j(\theta)$ denote the same vector, but excluding
terms from $\theta_j$, and $R_{jj'}(\theta)$ denote the 
same vector, but excluding both $\theta_j$ and $\theta_{j'}$. 

To aid in intuition, it will sometimes be useful to explicitly
write the inner product of $R(\theta)$ with a vector $G$ as
a sum of products of components of $\theta_j$.  Let $G$
be a $|R(\theta)|$-length vector with element $G_i$, and
let $R(\theta)_i$ denote the $i$th row of $R(\theta)$.
Then define

\begin{equation}\label{r_indexing_definition}
G^T R(\theta) = \sum_{i=1}^{|R(\theta)|} G_i R(\theta)_i 
           := \sum_{r \in R} G_r \prod_{j \in [J]} \theta_{jr_j}
\end{equation}

Here, we have ``overloaded'' the definition of $R$ to express the sum
over different products of $\theta$.  We define $r_j$ as the 
index of $\theta_j$ in the row of $R(\theta)$ corresponding to
$r_j$, $G_r$ as the element of $G$ corresponding to that index set,
and the sum over $r \in R$ as the sum over all rows.

We define the inner product of $G$ with $R_j(\theta)$ similarly,
only with the result as a $D_j$-length vector.  Specifically,
define

$$
\sum_{r \in R_j} G_r \prod_{k \in [J] \setminus j} \theta_{kr_k}
$$

such that $G_r$ is a $D_j$-sized column vector.  Finally, define

$$
\sum_{r \in R_{jj'}} G_r \prod_{k \in [J] \setminus \{j, j'\} } \theta_{kr_k}
$$

the same way, except where $G_r$ is a $D_j \times D_{j'}$ matrix.

This notation is intended to make it easy to express
sums of products of elements
of $\theta$ in such a way that no two terms from a single $\theta_j$
are multiplied together.
The value of this notation will hopefully become clear
in the following lemmas.

\begin{lemma}\label{lem:mfvb_factorization}
Suppose \eq{variational_exp_def} holds across all $j$; that is,
$$
p(\theta_{j} | \theta_{i \in [J]\setminus j}, x) = \exp(\npp_j^{T} \theta_{j} - A_j(\npp_j)).
$$
Then the posterior $p(\theta \vert x)$ can be written in the form
\begin{equation}\label{eq:mfvb_factorization}
\log p(\theta \vert x) =  \sum_{r \in R} G_r \prod_{j\in[J]} \theta_{jr_j} + \constant
\end{equation}
where the terms $G_r$ and $\constant$ are constant in all $\theta$
\footnote{Strictly speaking, $\constant$ is redundant since $\{\emptyset, ..., \emptyset\} \in R$.
Here and below, for additional clarity we will always write a constant.}.
\end{lemma}

\begin{proof}
We see that 
$\log p(\theta | x) = \log p(\theta_{j} | \theta_{i \in [J]\setminus j}, x) +
   \log p(\theta_{i \in [J]\setminus j} | x)$
depends on $\theta_j$ only via the first term in the sum.  By \eq{variational_exp_def},
$$
p(\theta_{j} | \theta_{i \in [J]\setminus j}, x) = \exp(\npp_j^{T} \theta_{j} - A_j(\npp_j))
$$
It follows that $\log p(\theta | x)$ is linear in the vector $\theta_j$.
But this is true for all $\theta_j$, and the above form for $p(\theta | x)$ follows.
\end{proof}

\begin{lemma}\label{lem:eta_definition}
Suppose \eq{variational_exp_def} holds across all $j$.
Then, for the natural parameter $\npp_{j}$, we have the following equations:

\begin{eqnarray}
  \npp_{j} &=& \sum_{r \in R_{j}} G_r
                    \prod_{k\in[J]\setminus j} \theta_{kr_k} \nonumber \\
  \npp_{j} &=& \frac{\partial \log p(\theta | x)}{\partial \theta_j} \nonumber \\
  H        &=& \frac{ \partial \npq }{ \partial \mpq^T }
            = E_{q^{*}} \left(\frac{ \partial^2  \log p(\theta | x)}
                       { \partial \theta \partial \theta^T}\right) \label{eq:h_definition}
\end{eqnarray}

Here, each $G_r$ is an $D_j$-length vector that is constant 
with respect to $\theta$.
\end{lemma}

\begin{proof}
The first result follows by collecting the terms for the $i$th component of $\theta_j$ in
\eq{mfvb_factorization} and applying Bayes' theorem.  The second is simply observing that
differentiating with respect to $\theta_j$ is a notationally tidy way
to collect the $j$ terms.

For the third result, recall from \eq{exp_approx_marg} that

\begin{eqnarray*}
\npq_{j} &=& \mbe_{q^{*}} [\npp_{j}] \\
  &=&  \mbe_{q^{*}} \left[ \sum_{r \in R_j} G_r
                    \prod_{k\in[J]\setminus j} \theta_{kr_k} \right]  \\
  &=& \sum_{r \in R_j} G_r \prod_{k\in[J]\setminus j} \mpq_{kr_k}  \Rightarrow \\
\frac{ \partial \npq_{j} }{ \partial \mpq_{j'}^T } &=&
    \sum_{r \in R_{jj'}} G'_r \prod_{k\in[J]\setminus \{j,j'\} } \mpq_{kr_k}\\
  &=& \mbe_{q^{*}} \left[ \sum_{r \in R_{jj'}} G'_r
               \prod_{l\in[J]\setminus \{j, j'\}} \theta_{kr_k} \right] \\
  &=& \mbe_{q^{*}} \left[
     \frac{\partial \log p(\theta | x)}{\partial \theta_{i} \partial \theta_{j'}^T} \right]
\end{eqnarray*}

Note that this proof relied on the fact that only one element of
$\theta_{j'}$ is in each product term, which allowed us to exchange
the derivative with respect to the expectation with the expectation of
the derivative with respect to $\theta_{j'}$.

\end{proof}

\subsection{Linear response} \label{app:lr_derivs}

We here derive the three equalities in \eqs{main_dm_dt}, \eqss{main_dM_dt},
and \eqss{main_dM_dm}, which appear
respectively as three propositions below. In these propositions,
we assume that $p(\theta | x)$
is in the exponential family as above. We will further assume
that all natural parameters (for $p$ or variational approximations)
are in the interior of the parameter space.  The fact that the natural
parameters are on the interior of the feasible space means that there exists an
open ball around them that is also feasible.  Let that ball have radius
$\delta$, and let $t$ be within a $\delta$ ball of the origin.  Then
$p_t(\theta | x)$ is well defined for $t$ in an open set containing zero.
These assumptions
will allow us to apply dominated convergence (cf.\ Section 2.3
of \cite{keener:2010:theoretical}).

\begin{proposition} \label{prop:dm_dt}
  $\frac{d}{dt} \mbe_{p_t} \theta = \truecov_{t}$.
\end{proposition}

\begin{proof}

\begin{align*} 
  \frac{d}{d t^{T}} \mbe_{p_{t}} \theta
    &= \frac{d}{d t^{T}} \int_{\theta}
          \theta e^{t^{T} \theta - c(t)}
        p(\theta | x) d\theta
        \quad \textrm{by the definition of $p_t$ in \eq{perturbed_dens}} \\
    &= \int_{\theta} \theta
          \left[
            \frac{d}{d t^{T}} e^{t^{T} \theta - c(t)}
          \right]
        p(\theta | x) d\theta
        \quad \textrm{by dominated convergence} \\
    &= \int_{\theta} \theta \theta^{T} e^{t^{T} \theta - c(t)} p(\theta | x) d\theta
      - \int_{\theta} \theta e^{t^{T} \theta - c(t)} p(\theta | x) d\theta
        \cdot \frac{d c(t)}{d t^{T}}
        \\
    &= \mbe_{p_t}\left[\theta \theta^T\right] -
       \mbe_{p_t}\left[\theta\right] \mbe_{p_t}\left[\theta\right]^T = \truecov_{t}
\end{align*}

\end{proof}

To approximate $\frac{d \mpq_{t}}{d t^{T}}$, we assume
not only that $\mpq_{t} \approx \mbe_{p_t} \theta$ for any particular $t$ but
further that $\mpq_t$ tracks the true mean $\mbe_{p_t} \theta$ as $t$ varies.
In this case, by \prop{dm_dt}, we have
$$
  \frac{d \mpq_{t}}{d t^{T}}
    \approx \frac{d}{d t^{T}} \mbe_{p_{t}} \theta
    = \truecov_{t},
$$
the first (approximate) equality in \eq{main_dm_dt}.

To derive the final two equalities in \eqs{main_dM_dt} and \eqss{main_dM_dm}, we make use of the following lemma.

\begin{lemma} \label{lem:dM_deta}
  $M_{t,j}$ depends on $t$ only via $\npq_{t,j}$, the natural parameter of the $q_{t,j}^{*}$ distribution.
  And $\frac{dM_{t,j}}{d\eta_{t,j}^T} = \Sigma_{q_{t,j}^{*}}$.
\end{lemma}

\begin{proof}
The first part of the lemma follows from writing the definition of $M_{t,j}$:
$$
  M_{t,j}
    = \mbe_{q_{t,j}^{*}} \theta_j
    = \int_{\theta_j}
            \theta_j \exp \left(
              \npq_{t,j}^{T} \theta_j - A_j\left(\npq_{t,j}\right)
            \right)
          d \theta_j.
$$

For the second part,
\begin{align*}
  \frac{d M_{t,j}}{d \npq_{t,j}^T}
      &= \int_{\theta_j}
          \frac{d}{d \npq_{t,j}^T}
            \theta_j \exp \left(
              \npq_{t,j}^{T} \theta_j - A_j\left(\npq_{t,j}\right)
            \right)
          d \theta_j \quad \textrm{by dominated convergence} \\
      &= \int_{\theta_j}
            \theta_j
              \left[ \theta_j^{T} - \mbe_{q_{t,j}^{*}} \theta_j^T \right]
              \exp \left(
                \npq_{t,j}^{T} \theta_j - A_j\left(\npq_{t,j}\right)
              \right)
          d \theta_j \\
      &= \Sigma_{q_{t,j}^{*}}
\end{align*}
\end{proof}

\begin{proposition}
$
  \frac{\partial M_t}{\partial t^T} = \vbcov_{t}.
$
\end{proposition}

\begin{proof}
By \lem{dM_deta}, we have for any indices $i$ and $j$ in $[J]$ that
\begin{equation} \label{eq:dM_dt}
  \frac{\partial M_{t,j}}{\partial t_i^T}
    = \frac{d M_{t,j}}{d \npq_{t,j}^T} \frac{\partial \npq_{t,j}}{\partial t_i^T},
\end{equation}
where the first factor is also given by \lem{dM_deta}.
It remains to find the second factor, $\frac{\partial \npq_{t,j}}{\partial t_i^T}$.
By the discussion
after \eq{variational_exp_def} and the construction of $p_t$,
the natural parameter $\npp_{t,j}$
of $p_{t}\left(\theta_{j} | \theta_{i \in [J]\setminus j}, x\right)$
satisfies
$$
  \npp_{t,j} = \sum_{r \in R_j} G_r
                    \prod_{k\in[J]\setminus j} \theta_{kr_k}  + t_j.
$$
So, as in the derivation of \eq{exp_approx_marg},
the natural parameter $\npq_{t,j}$ of $q_j^*(\theta_j)$ 
satisfies
\begin{equation} \label{eq:etat}
  \npq_{t,ji}
    = \sum_{r \in R_j} G_r \prod_{k\in[J]\setminus j} \mpq_{t,kr_k} + t_{j}
\end{equation}
for $\mpq_{t,r} := \mbe_{q^*_{t,r}} \theta_{r}$.

Let $d_j$ be the dimension of $\theta_j$ and hence the
dimension of $\npq_{t,j}$ and $t_j$. Hence,
$$
  \frac{\partial \npq_{t,j}}{\partial t_i^T}
    = \left\{ \begin{array}{ll}
        I_{d_j} & j = i \\
        0_{d_j,d_i} & \mathrm{else}
      \end{array} \right. ,
$$
where $I_{a}$ is the identity matrix of dimension $a$, and $0_{a,b}$ is the
all zeros matrix of dimension $a \times b$.

Finally, by \eq{dM_dt}, \lem{dM_deta}, and the expression for
$\frac{\partial \npq_{t,j}}{\partial t_i^T}$ just obtained, we have
$$
  \frac{\partial M_t}{\partial t^T} = \vbcov_{t} I_{D} = \vbcov_{t}.
$$
\end{proof}

\begin{proposition}
$
  \frac{d M_t}{d m_t^T} = \vbcov_{t} \frac{ \partial \eta_t }{ \partial m_t^T }.
$
\end{proposition}

\begin{proof}
By \lem{dM_deta} and analogous to \eq{dM_dt}, we have
\begin{equation} \label{eq:dM_dm}
  \frac{\partial M_{t,j}}{\partial m_{t,i}^T}
    = \frac{d M_{t,j}}{d \npq_{t,j}^T} \frac{\partial \npq_{t,j}}{\partial m_{t,i}^T}.
\end{equation}
The result follows immediately from \lem{dM_deta}.
\end{proof}

\section{Multivariate normal posteriors and SEM} \label{app:SEM}

For any target distribution $p(\theta | x)$, it is well-known that MFVB cannot be used
to estimate the covariances between the components of $\theta$.
In particular, if $q^*$ is the estimate
of $p(\theta | x)$ returned by MFVB,
$q^*$ will have a block-diagonal covariance matrix---no matter the form
of the covariance of $p(\theta | x)$. 
By contrast, the next result shows that the LRVB covariance estimate is exactly correct in the 
case where the target distribution, $p(\theta|x)$, is (multivariate) normal.

In order to prove this result, we will rely on the following lemma.
\begin{lemma} \label{lem:lrvb_mvn}
  Consider a target posterior distribution characterized by $p(\theta | x) = \gauss(\theta | \mu, \Sigma)$,
  where $\mu$ and $\Sigma$ may depend on $x$, and $\Sigma$ is invertible.
  Let $\theta = (\theta_{1}, \ldots, \theta_{J})$,
  and consider a MFVB approximation to $p(\theta| x)$ that factorizes as $q(\theta) = \prod_{j} q(\theta_j)$.
  Then the variational posterior means are the true posterior means; i.e. $m_j = \mu_j$ for all $j$ between 
  $1$ and $J$.
\end{lemma}

\begin{proof}
  The derivation of MFVB for the multivariate normal can be found in Section 10.1.2 of
  \cite{bishop:2006:pattern}; we highlight some key results here.
  Let $\Lambda = \Sigma^{-1}$. Let the $j$ index on a row or column correspond to $\theta_j$,
  and let the $-j$ index
  correspond to $\{\theta_{i}: i \in [J]\setminus j\}$. E.g., for $j=1$,
  $$
    \Lambda
      = \left[ \begin{array}{ll}
          \Lambda_{11} & \Lambda_{1,-1} \\
          \Lambda_{-1,1} & \Lambda_{-1,-1}
        \end{array} \right].
  $$
  By the assumption that $p(\theta | x) = \gauss(\theta | \mu, \Sigma)$, we have
\begin{equation}\label{eq:mvn_variational_dist}
    \log p(\theta_{j} | \theta_{i \in [J]\setminus j}, x)
      = -\frac{1}{2} (\theta_{j} - \mu_{j})^{T} \Lambda_{jj} (\theta_j - \mu_j) +
         (\theta_{j} - \mu_{j})^{T} \Lambda_{j,-j} (\theta_{-j} - \mu_{-j}) + \constant,
\end{equation}

  where the final term is constant with respect to $\theta_{j}$.
  It follows that
  \begin{align*}
    \log q^{*}_{j}(\theta_j)
      &= \mbe_{q^{*}_{i}: i \in [J]\setminus j} \log p(\theta, x) + \constant \\
      &= -\frac{1}{2} \theta_{j}^{T} \Lambda_{jj} \theta_j + \theta_j \mu_j \Lambda_{jj} - \theta_j \Lambda_{j,-j} (\mbe_{q^{*}} \theta_{-j} - \mu_{-j}).
  \end{align*}
  So 
  \begin{equation*}
    q^*_j(\theta_j) = \gauss(\theta_j | m_{j}, \Lambda_{jj}^{-1}),
  \end{equation*}
  with mean parameters
  \begin{equation} \label{eq:mvn_stable_point}
    m_{j} = \mbe_{q^{*}_j} \theta_j = \mu_{j} - \Lambda_{jj}^{-1} \Lambda_{j,-j} (m_{-j} - \mu_{-j})
  \end{equation}
  as well as an equation for $\mbe_{q^{*}} \theta^T \theta$.

Note that $\Lambda_{jj}$ must be invertible, for if it
were not, $\Sigma$ would not be invertible.

The solution $m = \mu$ is a unique stable point for
\eq{mvn_stable_point}, since the fixed point equations for each $j$
can be stacked and rearranged to give
\begin{eqnarray*}
m-\mu & = & -\left[\begin{array}{ccccc}
0 & \Lambda_{11}^{-1}\Lambda_{12} & \cdots & \Lambda_{11}^{-1}\Lambda_{1\left(J-1\right)} & \Lambda_{11}^{-1}\Lambda_{1J}\\
\vdots &  & \ddots &  & \vdots\\
\Lambda_{JJ}^{-1}\Lambda_{J1} & \Lambda_{JJ}^{-1}\Lambda_{J2} & \cdots & \Lambda_{JJ}^{-1}\Lambda_{J\left(J-1\right)} & 0
\end{array}\right]\left(m-\mu\right)\\
 & = & -\left[\begin{array}{ccccc}
\Lambda_{11}^{-1} & \cdots & 0 & \cdots & 0\\
\vdots & \ddots &  &  & \vdots\\
0 &  & \ddots &  & 0\\
\vdots &  &  & \ddots & \vdots\\
0 & \cdots & 0 & \cdots & \Lambda_{JJ}^{-1}
\end{array}\right]\left[\begin{array}{ccccc}
0 & \Lambda_{12} & \cdots & \Lambda_{1\left(J-1\right)} & \Lambda_{1J}\\
\vdots &  & \ddots &  & \vdots\\
\Lambda_{J1} & \Lambda_{J2} & \cdots & \Lambda_{J\left(J-1\right)} & 0
\end{array}\right]\left(m-\mu\right)\Leftrightarrow\\
0 & = & \left[\begin{array}{ccccc}
\Lambda_{11} & \cdots & 0 & \cdots & 0\\
\vdots & \ddots &  &  & \vdots\\
0 &  & \ddots &  & 0\\
\vdots &  &  & \ddots & \vdots\\
0 & \cdots & 0 & \cdots & \Lambda_{JJ}
\end{array}\right]\left(m-\mu\right) +\\
&& \left[\begin{array}{ccccc}
0 & \Lambda_{12} & \cdots & \Lambda_{1\left(J-1\right)} & \Lambda_{1J}\\
\vdots &  & \ddots &  & \vdots\\
\Lambda_{J1} & \Lambda_{J2} & \cdots & \Lambda_{J\left(J-1\right)} & 0
\end{array}\right]\left(m-\mu\right)\Leftrightarrow\\
0 & = & \Lambda \left(m-\mu\right) \Leftrightarrow\\
m & = & \mu.
\end{eqnarray*}
The last step follows from the assumption that $\Sigma$ (and hence $\Lambda$)
is invertible.  It follows that $\mu$ is the unique stable point of
\eq{mvn_stable_point}.

\end{proof}

\begin{proposition} \label{prop:lrvb_mvn}
  Assume we are in the setting of \lem{lrvb_mvn}, where additionally
  $\mu$ and $\Sigma$ are on the interior of the feasible parameter space.
  Then the LRVB covariance estimate exactly captures the true covariance,
  $\hat{\Sigma} = \Sigma$.

\end{proposition}

\begin{proof}

  Consider the perturbation for LRVB defined in \eq{perturbed_dens}.
  By perturbing the log likelihood, we change both the true means $\mu_t$
  and the variational solutions, $m_t$. The result is a valid
  density function since the original $\mu$ and $\Sigma$ are on the
  interior of the parameter space.
  By \lem{lrvb_mvn}, the MFVB solutions are exactly the true
  means, so $m_{t,j} = \mu_{t,j}$, and the derivatives are the same
  as well.  This means that the first term in \eq{lrvb_est} is
  not approximate, i.e.
  \begin{equation*}
  \frac{d \mpq_{t}}{d t^{T}}
    = \frac{d}{d t^{T}} \mbe_{p_{t}} \theta
    = \truecov_{t},
  \end{equation*}
  It follows from the arguments in \app{SEM} that the LRVB covariance
  matrix is exact, and $\hat{\Sigma} = \Sigma$.
  
\end{proof}

\subsection{Comparison with supplemented expectation-maximization}\label{subsec:SEM}

This result about the multivariate normal distribution
draws a connection between LRVB
corrections and the ``supplemented expectation-maximization'' (SEM)
method of \cite{meng:1991:using}.  SEM is an asymptotically
exact covariance correction for the EM algorithm that transforms
the full-data Fisher information matrix into the observed-data Fisher
information matrix using a correction that is formally similar to
\eq{lrvb_est}.  In this section, we argue that this similarity
is not a coincidence; in fact the SEM correction is an
asymptotic version of LRVB with two variational blocks,
one for the missing data and one for the unknown parameters.

Although LRVB as described here requires a prior 
(unlike SEM, which supplements the MLE),
the two covariance corrections coincide when
the full information likelihood is approximately log quadratic
and proportional to the posterior, $p(\theta \vert x)$.
This might be expected to occur when we have a large number
of independent data points informing each parameter---i.e.,
when a central limit theorem applies and the priors do not
affect the posterior.
In the full information likelihood, some
terms may be viewed as missing data, whereas in the Bayesian
model the same terms may be viewed as latent parameters,
but this does not prevent us from formally comparing the two methods.

We can draw a term-by-term analogy with
the equations in \cite{meng:1991:using}. We denote variables
from the SEM paper with a superscript ``$SEM$'' to avoid confusion.
MFVB does not differentiate between missing
data and parameters to be estimated, so our $\theta$ corresponds to
$(\theta^{SEM}, Y_{mis}^{SEM})$ in \cite{meng:1991:using}.
SEM is an asymptotic
theory, so we may assume that $(\theta^{SEM}, Y_{mis}^{SEM})$ have a
multivariate normal
distribution, and that we are interested in the mean and covariance of
$\theta^{SEM}$.

In the E-step of \cite{meng:1991:using}, we replace $Y_{mis}^{SEM}$ with
its conditional expectation given the data and other $\theta^{SEM}$.
This corresponds precisely to \eq{mvn_stable_point}, taking
$\theta_j = Y_{mis}^{SEM}$.  In the M-step, we find the maximum
of the log likelihood with respect to $\theta^{SEM}$, keeping
$Y_{mis}^{SEM}$ fixed at its expectation.  Since the mode
of a multivariate normal distribution is also its mean,
this, too, corresponds to \eq{mvn_stable_point}, now taking
$\theta_j = \theta^{SEM}$.

It follows that the MFVB and EM fixed point equations are the same;
i.e., our $M$ is the same as their $M^{SEM}$, and
our $\partial M / \partial m$ of \eq{dM_dm} corresponds
to the transpose of their $DM^{SEM}$, defined in \eqw{2.2.1}
of \cite{meng:1991:using}.  Since the ``complete information'' corresponds to
the variance of $\theta^{SEM}$ with fixed values for $Y_{OBS}^{SEM}$,
this is the same as our $\Sigma_{q^*,11}$, the variational covariance,
whose inverse is $I_{oc}^{-1}$.  Taken all together, this means that
equation (2.4.6) of \cite{meng:1991:using} can be
re-written as our \eq{lrvb_est}.
\begin{align*}
V^{SEM} =& I_{oc}^{-1} \left(I - DM^{SEM}\right)^{-1} \Rightarrow\\
\Sigma =& \vbcov \left(I - \left(\frac{\partial M}{\partial m^T}\right)^T \right)^{-1}
       = \left(I - \frac{\partial M}{\partial m^T} \right)^{-1} \vbcov
\end{align*}

\section{Multivariate normal mixture details} \label{app:mvn_details}

In this section we derive the basic formulas needed to calculate \eq{lrvb_est}
and \eq{influence_score} for a finite mixture of normals, which 
is the model used in \mysec{experiments}.  We will
follow the notation introduced in \mysec{normal_mix_model}.

Let each observation, $x_{n}$, be a $P\times1$ vector. We will denote
the $P$th component of the $n$th observation $x_{n}$, with a
similar pattern for $z$ and $\mu$. We will denote the $p$, $q$th
entry in the matrix $\Lambda_{k}$ as $\Lambda_{k,pq}$. The data
generating process is as follows:

\begin{eqnarray*}
\log P\left(x_{n}|z_{n},\mu,\Lambda\right) & = &
    \sum_{n=1}^{N}z_{nk}\log\phi_{k}(x_{n}) + \constant\\
\log\phi_{k}(x) & = & -\frac{1}{2}\left(x - \mu_{k}\right)^{T} \Lambda_{k}\left(x-\mu_{k}\right) + 
    \frac{1}{2}\log\left|\Lambda_{k}\right|+ \constant\\
\log P(z_{nk}|\pi_{k}) & = & \sum_{k=1}^{K}z_{nk}\log\pi_{k} + \constant\\
\log P(z,\mu,\pi,\Lambda | x) & = & \sum_{n=1}^{N}\sum_{k=1}^{K}z_{nk}\left(\log\pi_{k}-\frac{1}{2}\left(x_{n}-\mu_{k}\right)^{T}\Lambda_{k}\left(x_{n}-\mu_{k}\right)+\frac{1}{2}\log\left|\Lambda_{k}\right|\right)+ \constant
\end{eqnarray*}

In all our results we simply used
improper, flat priors, though it would be trivial to incorporate conjugate priors. 

From the assumptions in \eq{perturbation_assumptions},
the posterior expectation of $x^{*}$ will always have
$\mbe_q x^{*} = x$, so for notational convenience we can simply drop
the $*$ and apply the LRVB formulas as if $x$ were a random parameter.
However, it is useful to remember that the parameter $x$ is different
from the variable we condition on -- we are actually estimating
$p(\alpha, z, x^{*} | x)$.

The parameters $\mu_{k}$, $\Lambda_{k}$, $\pi$, and $z_{n}$ will
each be given their own variational distribution. By standard results,
the variational distributions will be:
\begin{eqnarray*}
q_{\mu_{k}} & = & \textrm{Multivariate Normal}\\
q_{\Lambda_{k}} & = & \textrm{Wishart}\\
q_{\pi} & = & \textrm{Dirichlet }\\
q_{z_{n}} & = & \textrm{Multinoulli (one multinomial draw)}\\
q_{x^*_{n}} & = & \textrm{Multivariate Normal}
\end{eqnarray*}

The sufficient statistics for $\mu_{k}$ are all terms of the form
$\mu_{kp}$ and $\mu_{kp}\mu_{kq}$. Consequently, the sub-vector
of $\theta$ corresponding to $\mu_{k}$ is 
\begin{eqnarray*}
\theta_{\mu_{k}} & = & \left(\begin{array}{c}
\mu_{k1}\\
\vdots\\
\mu_{kp}\\
\mu_{k1}\mu_{k1}\\
\mu_{k1}\mu_{k2}\\
\vdots\\
\mu_{kP}\mu_{kP}
\end{array}\right)
\end{eqnarray*}

We will only save one copy of $\mu_{kp}\mu_{kq}$ and $\mu_{kq}\mu_{kp}$,
so $\theta_{\mu_{k}}$ has length $P+\frac{1}{2}\left(P+1\right)P$.
For all the parameters, we denote the complete stacked vector without
a $k$ subscript:
\begin{eqnarray*}
\theta_{\mu} & = & \left(\begin{array}{c}
\theta_{\mu_{1}}\\
\vdots\\
\theta_{\mu_{K}}
\end{array}\right)
\end{eqnarray*}

The sufficient statistics for $x^*$ are analogous to those for $\mu$.

The sufficient statistics for $\Lambda_{k}$ are all the terms $\Lambda_{k,pq}$
and the term $\log\left|\Lambda_{k}\right|$. Again, since $\Lambda$
is symmetric, we do not keep redundant terms, so $\theta_{\Lambda_{k}}$
has length $1+\frac{1}{2}\left(P+1\right)P$.

The sufficient statistics for $\pi$ is the $K$-vector $\left(\log\pi_{1},...,\log\pi_{K}\right)$.

The sufficient statistics for $z$ is simply the $N\times K$ values
$z_{nk}$ themselves.

In terms of \mysec{influence}, we have
\begin{eqnarray*}
\alpha & = & \left(\begin{array}{c}
\theta_{\mu}\\
\theta_{\Lambda}\\
\theta_{\pi}
\end{array}\right)\\
z & = & \left(\begin{array}{c}
\theta_{z}\end{array}\right)\\
x & = & \left(  \theta_x \right)
\end{eqnarray*}

That is, we are primarily interested in the covariance of the sufficient
statistics of $\mu$, $\Lambda$, and $\pi$, $z$ are nuisance
parameters, and $x^*$ is the ``unobserved'' data.

To put the log likelihood in terms useful for LRVB, we must express
it in terms of the sufficient statistics, taking into account the
fact the $\theta$ vector does not store redundant terms (e.g. it
will only keep $\Lambda_{ab}$ for $a<b$ since $\Lambda$ is symmetric).

\begin{eqnarray*}
-\frac{1}{2}\left(x_{n}-\mu_{k}\right)^{T}\Lambda_{k}\left(x_{n}-\mu_{k}\right) & = & -\frac{1}{2}trace\left(\Lambda_{k}\left(x_{n}-\mu_{k}\right)\left(x_{n}-\mu_{k}\right)^{T}\right)\\
 & = & -\frac{1}{2}\sum_{a}\sum_{b}\left(\Lambda_{k,ab}\left(x_{n,a}-\mu_{k,a}\right)\left(x_{n,b}-\mu_{k,b}\right)\right)\\
 & = & -\frac{1}{2}\sum_{a}\sum_{b}\left(\Lambda_{k,ab}\mu_{k,a}\mu_{k,b}-\Lambda_{k,ab}x_{n,a}\mu_{k,b}-\Lambda_{k,ab}x_{n,b}\mu_{k,a}+\Lambda_{k,ab}x_{n,a}x_{n,b}\right)\\
 & = & -\frac{1}{2}\sum_{a}\Lambda_{k,aa}\left(\mu_{k}^{2}\right)^{a}+\sum_{a}\Lambda_{k,aa}x_{n,a}\mu_{k,a}-\frac{1}{2}\sum_{a}\Lambda_{k,aa}\left(x_{n}^{2}\right)^{2}-\\
 &  & \frac{1}{2}\sum_{a\ne b}\Lambda_{k,ab}\mu_{k,a}\mu_{k,b}+\sum_{a\ne b}\Lambda_{k,ab}x_{n,a}\mu_{k,b}-\frac{1}{2}\sum_{a\ne b}\Lambda_{k,ab}x_{n,a}x_{n,b}\\
 & = & -\frac{1}{2}\sum_{a}\Lambda_{k,aa}\left(\mu_{k}^{2}\right)^{a}+\sum_{a}\Lambda_{k,aa}x_{n,a}\mu_{k,a}-\frac{1}{2}\sum_{a}\Lambda_{k,aa}\left(x_{n}^{2}\right)^{2}-\\
 &  & \sum_{a<b}\Lambda_{k,ab}\mu_{k,a}\mu_{k,b}+\sum_{a<b}\Lambda_{k,ab}\left(x_{n,a}\mu_{k,b}+x_{n,b}\mu_{k,a}\right)-\sum_{a<b}\Lambda_{k,ab}x_{n,a}x_{n,b}
\end{eqnarray*}

The MFVB updates and covariances in $V$ are all given by properties of standard
distributions. To compute the LRVB corrections, it only remains to
calculate the hessian of $H$. These terms can be read directly off the posterior.
First we calculate derivatives with respect to components of $\mu$.

\begin{eqnarray*}
\frac{\partial^{2}H}{\partial\mu_{k,a}\partial\Lambda_{k,ab}} & = & \sum_{i}z_{nk}x_{n,b}\\
\frac{\partial^{2}H}{\partial\left(\mu_{k,a}\mu_{k,b}\right)\partial\Lambda_{k,ab}} & = & -\left(\frac{1}{2}\right)^{1(a=b)}\sum_{n}z_{nk}\\
\frac{\partial^{2}H}{\partial\mu_{k,a}\partial z_{nk}} & = & \sum_{b}\Lambda_{k,ab}x_{n,b}\\
\frac{\partial^{2}H}{\partial\left(\mu_{k,a}\mu_{k,b}\right)\partial z_{nk}} & = & -\left(\frac{1}{2}\right)^{1(a=b)}\Lambda_{k,ab}
\end{eqnarray*}

All other $\mu$ derivatives are zero. For $\Lambda$,

\begin{eqnarray*}
\frac{\partial^{2}H}{\partial\Lambda_{k,ab}\partial z_{nk}} & = & -\left(\frac{1}{2}\right)^{1(a=b)}\left(x_{n,a}x_{n,b}-\mu_{k,a}x_{n,b}-\mu_{k,b}x_{n,a}+\mu_{k,a}\mu_{k,b}\right)\\
\frac{\partial^{2}H}{\partial\log\left|\Lambda_{k}\right|\partial z_{nk}} & = & \frac{1}{2}
\end{eqnarray*}

The remaining $\Lambda$ derivatives are zero. The only nonzero second
derivatives for $\log\pi$ are to $Z$ and are given by 
\begin{eqnarray*}
\frac{\partial^{2}H}{\partial\log\pi_{j}\partial z_{nk}} & = & 1
\end{eqnarray*}

To calculate the influence scores, we additionally need second derivatives
involving $x$.

\begin{eqnarray*}
\frac{\partial^{2}H}{\partial x_{n,a}\partial\mu_{k,b}} & = & z_{nk}\Lambda_{k,ab}\\
\frac{\partial^{2}H}{\partial x_{n,a}\partial\Lambda_{k,ab}} & = & z_{nk}\mu_{k,b}\\
\frac{\partial^{2}H}{\partial x_{n,a}\partial z_{nk}} & = & \sum_{b}\mu_{k,b}\Lambda_{k,ab}\\
\frac{\partial^{2}H}{\partial\left(x_{n,a}x_{n,b}\right)\partial\Lambda_{k,ab}} & = & -z_{nk}\left(\frac{1}{2}\right)^{1\left(a=b\right)}\\
\frac{\partial^{2}H}{\partial\left(x_{n,a}x_{n,b}\right)\partial z_{nk}} & = & -\Lambda_{k,ab}\left(\frac{1}{2}\right)^{1(a=b)}
\end{eqnarray*}

All other second derivatives involving $x$ are zero.
Note in particular that $H_{zz} = 0$, allowing
efficient calculation of \eq{nuisance_lrvb_est}. 

\section{Influence scores} \label{app:influence_details}

\subsection{Influence score derivations}

In this section, we derive the formulas in \mysec{influence}.  We
will follow the notation there defined.
Consider the ``influence score'' given by the derivative of
the conditional expectation:

\begin{eqnarray*}
m_{\theta_{i}}\left(x_{n}\right) & = &
  \mbe_{p}\left[\theta_{i}\vert x_{1},...,x_{n},...,x_{N}\right]\\
\frac{d}{dx_{n}}m_{\theta_{i}}\left(x_{n}\right) &:=& m_{\theta_{i}}'\left(x_{n}\right)
\end{eqnarray*}

A Taylor expansion of $m_{\theta_{i}}\left(x^*_{n}\right)$ around $x_{n}$ gives
\begin{eqnarray*}
m_{\theta_{i}}\left(x_{n}^{*}\right) 
&= & m_{\theta_{i}} \left(x_{n}\right) + m_{\theta_{i}}' \left(x_{n} \right)
             \left(x_{n}^{*} - x_{n}\right) + \\
  && \quad\quad O\left(\left(x_{n}^{*}-x_{n}\right)^{2}\right)
\end{eqnarray*}
Multiplying both sides by $\left(x_{n}^{*}-x_{n}\right)$ and taking
expectations conditional on $x$ gives
\begin{align*}
&\mbe\left[ \left(m_{\theta_{i}} \left(x_{n}^{*}\right) - m_{\theta_{i}}\left(x_{n}\right)\right)
           \left(x_{n}^{*} - x_{n}\right)\vert x\right] \\ 
&\quad = m_{\theta_{i}}'\left(\theta_{i}, x_{n}\right)
           \mbe\left[\left(x_{n}^{*} - x_{n}\right)^{2}\vert x\right] +  O\left(\left(x_{n}^{*}-x_{n}\right)^{3}\right)
\end{align*}
On the left side,
\begin{align*}
&\mbe\left[\left(m_{\theta_{i}}\left(x_{n}^{*}\right) - m_{\theta_{i}}\left(x_{n}\right)\right)
          \left(x_{n}^{*}-x_{n}\right)\vert x\right]  \\
&\quad =  \mbe\left[\mbe\left[\left(m_{\theta_{i}}\left(x_{n}^{*}\right) - m_{\theta_{i}}\left(x_{n}\right)\right)
    \left(x_{n}^{*} - x_{n}\right)\vert x_{n}^{*}\right] \vert x\right]\\
&\quad =  \mbe\left[\mbe\left[\left(\mbe\left[\theta_{i}\vert x_{1},.,x_{n}^{*},.,x_{N}\right] -
      m_{\theta_{i}}\left(x_{n}\right)\right)
      \left(x_{n}^{*}-x_{n}\right)\vert x_{n}^{*}\right]\vert x\right]\\
&\quad = \mbe\left[\mbe\left[\left(\theta_{i} - m_{\theta_{i}} \left(x_{n}\right)\right)
                     \left(x_{n}^{*} - x_{n}\right)\vert x_{n}^{*}\right]\vert x\right]\\
&\quad =  \mbe\left[\left(\theta_{i} - m_{\theta_{i}}\left(x_{n}\right)\right)
              \left(x_{n}^{*} - x_{n}\right)\vert x\right]\\
&\quad = Cov\left(\theta_{i}, x_{n}^{*}\vert x\right)
\end{align*}

On the right side,
\begin{align*}
&m_{\theta_{i}}' \left(\theta_{i}, x_{n}\right)
    \mbe\left[\left(x_{n}^{*} - x_{n}\right)^{2}\vert x\right] + 
    O\left(\left(x_{n}^{*}-x_{n}\right)^{3}\right) \\
&\quad = m_{\theta_{i}}'\left(\theta_{i}, x_{n}\right)\epsilon + O\left(\epsilon^{\frac{3}{2}}\right)
\end{align*}

So that
\begin{eqnarray*}
m_{\theta_{i}}'\left(\theta_{i}, x_{n}\right) & = &
    \frac{1}{\epsilon}Cov\left(\theta_{i}, x_{n}^{*}\vert x\right) +
    O\left(\epsilon^{\frac{3}{2}}\right)
\end{eqnarray*}
...which is \eq{influence_as_derivative}.

We now assume that our parameter space can be divided into types of
variables: $\alpha$ and $z$ as before, and $x$, the perturbed data.
As before, we also assume that each has its own variational distribution.

\begin{eqnarray*}
\theta & = & \left(\begin{array}{c}
\alpha\\
x\\
z
\end{array}\right)\\
\Sigma & = & \left[\begin{array}{ccc}
\Sigma_{\alpha} & \Sigma_{\alpha x^{*}} & \Sigma_{\alpha z}\\
\Sigma_{x^{*}\alpha} & \Sigma_{x^{*}} & \Sigma_{x^{*}z}\\
\Sigma_{z\alpha} & \Sigma_{zx^{*}} & \Sigma_{z}
\end{array}\right]
\end{eqnarray*}

As before, we use a similar partition for $V$ and $H$. Specifically,
\begin{eqnarray*}
V & = & \left[\begin{array}{ccc}
V_{\alpha} & 0 & 0\\
0 & V_{x^{*}} & 0\\
0 & 0 & V_{z}
\end{array}\right]\\
H & = & \left[\begin{array}{ccc}
H_{\alpha} & H_{\alpha x^{*}} & H_{\alpha z}\\
H_{x^{*}\alpha} & H_{x^{*}} & H_{x^{*}z}\\
H_{z\alpha} & H_{zx^{*}} & H_{z}
\end{array}\right]
\end{eqnarray*}

We are interested in $\Sigma_{\alpha x^{*}}$, the covariance between
$\alpha$ and $x$, which can be interpreted as influence scores.

The matrix $\Sigma_{x^{*}}$ is the result of an infinitesimal perturbation,
and so will be nearly zero. We will write
\begin{eqnarray*}
\Sigma_{x^{*}} & = & \epsilon S_{x^{*}}
\end{eqnarray*}
Note that $S_{x^{*}}$ is not necessarily diagonal if the variational
distribution for each datapoint $x_{n}$ is multidimensional. Applying
formula \eq{nuisance_lrvb_est} to eliminate $z$, we have
\begin{eqnarray*}
\left[\begin{array}{cc}
\Sigma_{\alpha} & \Sigma_{\alpha x^{*}}\\
\Sigma_{x^{*}\alpha} & \Sigma_{x^{*}}
\end{array}\right] & = & \left[\left(\begin{array}{cc}
I_{\alpha}-V_{\alpha}H_{\alpha} & -V_{\alpha}H_{\alpha x^{*}}\\
-\epsilon S_{x^{*}}H_{x^{*}q} & I_{x^{*}}
\end{array}\right)-\left(\begin{array}{c}
V_{\alpha}H_{\alpha z}\\
\epsilon S_{x^{*}}H_{x^{*}z}
\end{array}\right)Q_{z}\left(\begin{array}{cc}
V_{z}H_{z\alpha} & V_{z}H_{zx^{*}}\end{array}\right)\right]^{-1}\left(\begin{array}{cc}
V_{\alpha} & 0\\
0 & \epsilon S{}_{x^{*}}
\end{array}\right)\\
 & = & \left[\left(\begin{array}{cc}
I_{\alpha}-V_{\alpha}H_{\alpha} & -V_{\alpha}H_{\alpha x^{*}}\\
-\epsilon S_{x^{*}}H_{x^{*}q} & I_{x^{*}}
\end{array}\right)-\left(\begin{array}{cc}
V_{\alpha}H_{\alpha z}Q_{z}V_{z}H_{z\alpha} & V_{\alpha}H_{\alpha z}Q_{z}V_{z}H_{zx^{*}}\\
\epsilon S_{x^{*}}H_{x^{*}z}Q_{z}V_{z}H_{z\alpha} & \epsilon S_{x^{*}}H_{x^{*}z}Q_{z}V_{z}H_{z\alpha}
\end{array}\right)\right]^{-1}\left(\begin{array}{cc}
V_{\alpha} & 0\\
0 & \epsilon S{}_{x^{*}}
\end{array}\right)\\
 & = & \left[\begin{array}{cc}
I_{\alpha}-V_{\alpha}H_{\alpha}-V_{\alpha}H_{\alpha z}Q_{z}V_{z}H_{z\alpha} & -\left(V_{\alpha}H_{\alpha x^{*}}+V_{\alpha}H_{\alpha z}Q_{z}V_{z}H_{zx^{*}}\right)\\
-\epsilon\left(S_{x^{*}}H_{x^{*}q}+S_{x^{*}}H_{x^{*}z}Q_{z}V_{z}H_{z\alpha}\right) & I_{x^{*}}-\epsilon S_{x^{*}}H_{x^{*}z}Q_{z}V_{z}H_{z\alpha}
\end{array}\right]^{-1}\left(\begin{array}{cc}
V_{\alpha} & 0\\
0 & \epsilon S{}_{x^{*}}
\end{array}\right)\\
 & = & \left[\begin{array}{cc}
I_{\alpha}-V_{\alpha}H_{\alpha}-V_{\alpha}H_{\alpha z}Q_{z}V_{z}H_{z\alpha} & -Q_{\alpha x^{*}}\\
-\epsilon Q_{x^{*}\alpha} & I_{x^{*}}-\epsilon Q_{x^{*}}
\end{array}\right]^{-1}\left(\begin{array}{cc}
V_{\alpha} & 0\\
0 & \epsilon S{}_{x^{*}}
\end{array}\right)\\
\end{eqnarray*}

In the last step we have defined a some placeholder matrices
called $Q$ to simplify subsequent expressions:

\begin{eqnarray*}
Q_{z} & := & \left(I_{z}-V_{z}H_{z}\right)^{-1}\\
Q_{\alpha x^{*}} & := & V_{\alpha}H_{\alpha x^{*}}+V_{\alpha}H_{\alpha z}Q_{z}V_{z}H_{zx^{*}}\\
Q_{x^{*}\alpha} & := & S_{x^{*}}H_{xq}+S_{x^{*}}H_{x^{*}z}Q_{z}V_{z}H_{z\alpha}\\
Q_{x^{*}} & := & S_{x^{*}}H_{x^{*}z}Q_{z}V_{z}H_{z\alpha}
\end{eqnarray*}

This may appear to be a complicated expression, but it can be considerably
simplified by using the fact that $\Sigma_{\alpha x^{*}}\propto\epsilon$
and $\epsilon\approx0$, which allows us to eliminate all $\epsilon$
terms that are second-order or higher. We can also use the Taylor
expansion of the matrix inverse that gives, for $\epsilon$ small,
and invertible matrix $A$,
\begin{eqnarray*}
\left(I-\epsilon B\right)^{-1} & = & I+\epsilon B+O\left(\epsilon^{2}\right)
\end{eqnarray*}
Again applying a Schur complement, we can write the expression for
the upper-left corner:
\begin{eqnarray*}
\Sigma_{\alpha} & = & \left(I_{\alpha}-V_{\alpha}H_{\alpha}-V_{\alpha}H_{\alpha z}Q_{z}V_{z}H_{z\alpha}-\epsilon Q_{\alpha x^{*}}\left(I_{x^{*}}-\epsilon Q_{x^{*}}\right)^{-1}Q_{x^{*}\alpha}\right)^{-1}V_{\alpha}\\
 & = & \left(I_{\alpha}-V_{\alpha}H_{\alpha}-V_{\alpha}H_{\alpha z}Q_{z}V_{z}H_{z\alpha}\right)^{-1}V_{\alpha}+O\left(\epsilon\right)
\end{eqnarray*}
Note that as $\epsilon\rightarrow0$, this gives the ordinary LRVB
estimate for $\Sigma_{\alpha}$, as expected. Infinitesimal perturbations
to our data do not change our beliefs about the posterior covariance.
Next, the Schur complement formula for the upper right corner gives
\begin{eqnarray*}
\Sigma_{\alpha x^{*}} & = & \epsilon\Sigma_{\alpha}^{-1}\left(Q_{\alpha x^{*}}\left(I_{x^{*}}-\epsilon Q_{x^{*}}\right)^{-1}\right)S_{x^{*}}\\
 & = & \epsilon\Sigma_{\alpha}^{-1}\left(Q_{\alpha x^{*}}\left(I_{x^{*}}+\epsilon Q_{x^{*}}+O\left(\epsilon^{2}\right)\right)\right)S_{x^{*}}\\
 & = & \epsilon\Sigma_{\alpha}^{-1}Q_{\alpha x^{*}}S_{x^{*}}+O\left(\epsilon^{2}\right)\\
 & = & \epsilon\Sigma_{\alpha}^{-1}\left(V_{\alpha}H_{\alpha x^{*}}+V_{\alpha}H_{\alpha z}\left(I_{z}-V_{z}H_{z}\right)^{-1}V_{z}H_{zx^{*}}\right)S_{x^{*}}+O\left(\epsilon^{2}\right)
\end{eqnarray*}

Taking limits gives \eq{influence_score}.
\end{document}